\documentclass{ecai} 

\usepackage{latexsym}
\usepackage{amssymb}
\usepackage{amsmath}
\usepackage{amsthm}
\usepackage{booktabs}
\usepackage{enumitem}
\usepackage{graphicx}
\usepackage{color}
\usepackage{amsmath}
\usepackage{amssymb}
\usepackage{mathtools}
\usepackage{amsthm}
\usepackage{float}
\usepackage{algorithm}
\usepackage{algorithmic}
\usepackage[hidelinks]{hyperref}
\newcommand{\X}{\boldsymbol{X}}
\newcommand{\x}{\mathbf{x}}
\newcommand{\y}{\mathbf{y}}

\newcommand{\cov}{\text{Cov}}
\newcommand{\un}{1\!\!1}
\newcommand{\bulle}{\textcolor[rgb]{0.00,0.00,1.00}{\bullet}}
\usepackage{microtype}
\usepackage{graphicx}
\usepackage{subfigure}
\usepackage{booktabs} 
\usepackage[table]{xcolor} 

\newtheorem{theorem}{Theorem}

\newtheorem{proposition}[theorem]{Proposition}

\newtheorem{definition}{Definition}
\theoremstyle{plain}

\theoremstyle{definition}

\theoremstyle{remark}

\newcommand{\BibTeX}{B\kern-.05em{\sc i\kern-.025em b}\kern-.08em\TeX}

\begin{document}


\begin{frontmatter}


\paperid{123} 


\title{KNN and K-means in Gini Prametric Spaces}


\author[A]{\fnms{Cassandra}~\snm{Mussard}\orcid{0009-0008-8679-0088}\footnote{Equal contribution.}}
\author[B]{\fnms{Arthur}~\snm{Charpentier}\orcid{0000-0003-3654-6286}\footnotemark}
\author[C]{\fnms{Stéphane}~\snm{Mussard}\orcid{0000-0003-2331-630X}\thanks{Corresponding Author. Email: stephane.mussard@unimes.fr}\footnotemark} 

\address[A]{CNES}
\address[B]{UQAM}
\address[C]{Univ. Nîmes CHROME and UMP6 AIRESS}


\begin{abstract}
This paper introduces innovative enhancements to the K-means and K-nearest neighbors (KNN) algorithms based on the concept of Gini prametric spaces (as opposed to metric spaces with standard distance properties). Unlike traditional distance metrics, Gini prametrics incorporate both value-based and rank-based measures, offering robustness to noise and outliers. The main contributions of the paper include (1) proposing a Gini prametric that captures rank information alongside value distances, (2) presenting a Gini K-means algorithm that is proven to converge and demonstrates resilience to noisy data, and (3) introducing a Gini KNN method that rivals state-of-the-art approaches like Hassanat’s distance in noisy environments. Experimental evaluations on 16 datasets from the UCI repository reveal the superior performance and efficiency of Gini-based algorithms in clustering and classification tasks. This work opens new directions for rank-based prametrics in machine learning and statistical analysis. 
\end{abstract}

\end{frontmatter}


\section{Introduction}

Machine learning, data analysis, and pattern recognition are closely related fields, each encompassing a variety of tasks and approaches. Among the most common tasks are supervised classifications, where models are trained using labeled data, and unsupervised learning, such as clustering, which deals with finding patterns in unlabeled data. In each field, numerous algorithms are available based on different metrics, however some distance metrics are used in both contexts, classification and clustering. For instance, K-nearest neighbors or KNN (\citet{Fix1951,Cover1967}) and K-means algorithms (\citet{Steinhaus1957,Lloyd1982, Ball1965,Macqueen1967}) were, at the moment of their creation, based on the Euclidean distance, before moving to a more flexible metric \textit{i.e.} the Minkowski $p$-distance (\citet{Oti2021}). 

On the one hand, K-means involves assigning points to clusters by iteratively measuring the distance between the points and the centroids of the clusters, with centroids being updated at each step of the algorithm until convergence. \citet{Jain2008,Jain2010} shows that more than one thousand algorithms are available for K-means clustering. The developments of K-means are either concerned with initialization, convergence and partitioning the data (see for instance \citet{Pelleg1999,Steinbach2000}) or the distance metric to be used to measure the proximity or the similarity between points and centroids, see for example the Kernel K-means to incorporate non-linearity in distances (\citet{Scholkopf1998}), the Fuzzy C-means (\citet{Eschrich2003}), the K-medoids based on the median (\citet{Kaufman2005}), the X-means with Akaike or Bayesian information criteria (\citet{Pelleg2000}), and recently Fair K-means (\citet{Bateni2024}) including fairness in distance functions. 

On the other hand, KNN allows assigning points of the testing dataset to labels by measuring the distance between these test points and the K labeled points in the training dataset that are closest. The algorithm is therefore dependent on the number of neighbors K, but also, and more importantly, dependent on the choice of the distance (\citet{Syriopoulos2023}). \citet{Prasatha2019} show that non-convex distances are well suited for managing noise in data whereas \citet{Vincent2001} show that convex distances defined on non-linear manifolds may be a good strategy to refine the classification. As in the K-means, weighting the points differently (\citet{Gini1912}) or by means of fuzzification may improve the accuracy of the  algorithm   (\citet{Derrac2014fuzzy}). However, as mentioned in \citet{Prasatha2019}, comparing 44 distances brings out 8 distances that outperform the others on many datasets (see Section \ref{section:state-of-the-art} for details).  

Prametric (or premetric) spaces have been introduced by \citet{Arkhangelskii1990}. From our knowledge, apart from the use of non-convex distances, prametrics have not been employed in classification. 
A prametric space is a topological space that is more general than metric spaces, requiring neither symmetry nor indistinguishability, nor the validity of triangular inequality. It is provided with a function $d$ that is more general than a ``distance” in the usual sense, since $d$ must satisfy non-negativity $d(\x,\y) \geq 0$ and $d(\x,\x) = 0$. $d$ is then said to be a ``prametric", and since prametrics are not metrics (they usually do not satisfy the triangle inequality), they do not induce metric spaces as usual in topology. Indeed, the topology of a space induced by a prametric is necessarily sequential. 

The Gini index is often used in multivariate statistics because it is a rank-dependent statistics robust to noise and outliers. This robustness holds true for many tasks such as regressions (\citet{Yitzhaki2004,Yitzhaki2013}) or compression (\citet{charpentier2021principal}). We show in this paper that a \textit{prametric Gini} can be defined with a hyper-parameter and employed for supervised classification and clustering to deal with noisy data. The main contributions are:

1. To propose a Gini prametric based both on rank features and values of the points to capture two kinds of information, in particular the rank vectors for robustness.

2. To show that the proposed Gini K-means algorithm is convergent and robust to noise.

3. To show that Gini KNN may compete with KNN based on Hassanat distance, which is considered to be the most robust to noise.

The paper is structured as follows: Section \ref{section:state-of-the-art} introduces
existing works related to distance metrics for K-means and KNN algorithms. Section \ref{section:Gini_prametric} presents the new Gini prametric proposed and investigated in our work. Section
\ref{section:clustering} introduces Gini K-means and Gini KNN. Section \ref{exp} is dedicated to the evaluation protocol and outlines the
performance of the Gini prametric compared to existing approaches.
Section \ref{section:conclusion} discusses our findings before mentioning 
perspectives.


\section{State of the art}\label{section:state-of-the-art}
Metric distances play a central role in both supervised and unsupervised learning methods. The choice of the distance is challenging in practice, because of their  variety and also because of the difficulty of determining in advance which distance is the most appropriate for a given dataset. However, \citet{Prasatha2019} demonstrate that two distances outperform other ones (among a set of 44 distances) when dealing with noisy datasets and the K-nearest neighbor (KNN) classifier: the Hassanat distance and distances based on the L1 norm (Manhattan or city block distance), such as the Lorentzian and the Canberra distances. Other studies on K-means show that robust clustering may be performed thanks to weighted Manhattan distances (\citet{An2024}), quantum Manhattan distances (\citet{Wu2022}).

\begin{small}    
\begin{table}[H]
    \centering
    \begin{tabular}{cl}
    \hline\hline
       Minkowski (F1) & \\ \cline{1-1}
       L$p$ distance & $\sqrt[p]{ \sum_{j=1}^d |x_{j} - y_{j}|^p }$   \\
       Manhattan & $\sum_{j=1}^d |x_{j} - y_{j}|$
       \\ \hline\hline
        L1 distances (F2) &  \\ \cline{1-1}
        Lorentzian & $\sum_{j=1}^d \ln(1+|x_{j} - y_{j}|)$\\ 
         Canberra & $ \sum_{j=1}^d \displaystyle\frac{|x_{j} - y_{j}|}{|x_{j}| + |y_{j}|}$ \\ \hline\hline
          Inner product (F3)& \\ \cline{1-1}
         Cosine distance & $1-\displaystyle\frac{\sum_{j=1}^d x_jy_j}{\sqrt{\sum_{j=1}^dx_j^2}\sqrt{\sum_{j=1}^dy_j^2}}$\\ \hline\hline
          Squared Chord (F4) & \\ \cline{1-1} Hellinger
         & $\sqrt{2\sum_{j=1}^d(\sqrt{x_j} - \sqrt{y_j})^2}$ \\ \hline \hline
         Squared L2 (F5) & \\ \cline{1-1}
        Pearson Chi squared & $\displaystyle\sum_{j=1}^d \displaystyle\frac{(x_{j} - y_{j})^2}{y_{j}^2}$ 
         \\ Squared Chi distance & $\displaystyle\sum_{j=1}^d \displaystyle\frac{(x_{j} - y_{j})^2}{|x_{j} + y_{j}|}$ \\ \hline\hline
          Shannon entropy (F6)& \\ \cline{1-1}
         Jensen-Shannon  & $0.5\Big(\sum_{j=1}^d x_j\ln\displaystyle\frac{2x_{j}}{x_{j}+y_{j}}$ \\
         &$+\sum_{j=1}^d y_j\ln\displaystyle\frac{2y_{j}}{x_{j}+y_{j}}\Big)$ \\ \hline\hline
          Vicissitude (F7)& \\ \cline{1-1} Vicis Symmetric
         & $\displaystyle\sum_{j=1}^d \displaystyle\frac{(x_{j} - y_{j})^2}{\min(x_{j},y_{j})^2}$ \\ \hline\hline
          Others (F8) & \\ \cline{1-1}
         Hassanat & \footnotesize{if $\min(x_j,y_j) \geq 0$:}  \\
         &$1-\displaystyle\frac{1+\min(x_j,y_j)}{1+\max(x_j,y_j)}$ {else:} \\ 
          \multicolumn{2}{r}{$1-\displaystyle\frac{1+\min(x_j,y_j)+|\min(x_j,y_j)|}{1+\max(x_j,y_j)+|\min(x_j,y_j)|}$} \\ \hline\hline
    \end{tabular}
    \caption{Eight Families of distances in $\mathbb{R}^d$.}
    \label{tab:distances}
\end{table}    
\end{small}

A distance satisfies the following properties: non-negativity, symmetry, nullity if and only if the points are equal, and finally the triangle inequality. Apart from \citet{Elkan2003} who shows that the triangle inequality may serve to accelerate the K-means algorithm, there is no justification for all four of these properties to be satisfied at the same time. It might be possible to resort to pseudo-distances or prametrics to look for more flexible functional forms. From the review made by \citet{Prasatha2019}, eight families of distances can be identified in the literature. Let $\x,\y \in \mathbb R^d$ be two points, we provide in Table \ref{tab:distances} the eight families with one member of each family (or two)  that receives attention from the literature. 

In particular, \citet{Prasatha2019} show that Hassanat's distance  \cite{Hassanat2014} outperforms 43 other distances of the eight families based on the average precision and recall over many datasets of the UCI machine learning repository. They perform the same experiment with noise (from 10\% to 90\% of noise in the data with increment of 10\%) and find that the same ranking is obtained in terms of averaged recalls and precisions: 1) Hassanat (F8), 2) Lorentzian (F2), 3) Canberra (F2). The F1 and F3 distance families are less efficient, particularly the Euclidean and Cosine distances, as well as the Pearson distance (F8) based on Pearson's correlation coefficient.

\section{Gini prametric spaces }\label{section:Gini_prametric}
\subsection{Variance, Covariance, Norm and Distance}

Given a vector $\x\in\mathbb{R}^d$, its L2 norm is simply $\|\x\|_2=\displaystyle{\sqrt{\x^\top\x}}$, induced by the natural inner product in $\mathbb{R}^d$, $\langle\x,\y\rangle = \x^\top\y$. And the associated L2 distance is $d_2(\x,\y)=\|\x-\y\|_2$, but if neither first order homogeneity nor the triangle equality are need, we could simply consider the prametric $d_2(\x,\y)^2$, i.e., $(\x-\y)^\top(\x-\y)$. In statistics, the norm is associated with the ``standard deviation'', and the inner product is the standard Pearson's linear correlation, for centered vectors. For a centered random vector $\mathbf{X}$, one can define its Pearson-covariance matrix as $\text{Cov}(\X)=\mathbb{E}[\X^\top\X]$.

Nevertheless, it is well established that those measures are sensitive to outliers. A classical strategy is then to move from the L2 norm $\|\x\|_2$ to the L1 norm $\|\x\|_1=|\x|^\top\boldsymbol{1}$. E.g., instead of using the average (solution of $\text{argmin}\lbrace d_2(\x,m)^2\rbrace$), one could consider the median (solution of $\text{argmin}\lbrace d_1(\x,m)\rbrace$). And the L1 norm is related to ranks, since, if $x\in\mathbb{R}$ and $X$ denotes a real-valued random variable with c.d.f. $F$, $d_2(x,X)^2=\mathbb{E}[(x-X)^2]$ while $d_1(x,X)=\mathbb{E}[|x-X|]$. Then
 $$
 \nabla_x d_1(x,X) =\mathbb{E}[\text{sign}(x-X)]=2F(x)-1,
 $$
if $F$ is absolutely continuous, while, if $\x=(x_1,\cdots,x_n)$,
 $$
 \nabla_x d_1(x,\x) =2\widehat{F}_n(x)-1,
 $$
when the empirical e.c.d.f. is based on ranks (if $\x$ is ordered, $\widehat{F}_n(x)=R_{\x}(x_i)/n$ where $x_i\leq x\leq x_{i+1}$, and where $R_{\x}(x_i)$ denotes the rank of $x_i$ among $\x=(x_1,\cdots,x_n))$.
The use of ranks has been intensively considered in statistics to avoid the influence of outliers, as discussed in \citet{kendall1948rank,hollander1973nonparametric,huber1981robust}. Hence, \citet{spearman1904proof} suggested to use Spearman's covariance, where ranks are considered, instead of numerical values. More recently, \citet{Carcea2015} and \citet{Shelef2019} show the robustness of the Gini-covariance (see subsection \ref{sect:Gini-cov}) that can be used for the Gini prametric.

\subsection{Gini Index and Gini-Covariance}

The popular Gini index is defined as a normalized version of ``Gini Mean Difference'', \citet{Gini1912}, defined as the expected absolute difference between two randomly drawn observations from the same population. The later can be rewritten with the covariance operator, $$GMD(X) = 4\cov(X,F(X)).$$
\citet{Schechtman1987} introduced a Gini-covariance operator, defined, for a pair of variables, as
$$
GMD(X,Y)=4\cov(X,F_Y(Y))=2\cov(X,G_Y(Y))
$$
where $G_Y(Y)=2F_Y(Y)-1$ is centered, and therefore
$$
GMD(X,Y)=2\mathbb{E}[XG_Y(Y)].
$$
This dispersion measure, $GMD$, involves the usual covariance of one of the variables with the rank of the other one, as a compromise between Pearson covariance $\cov(X,Y)$ and the Spearman version, based on pure ranks, $\cov(F_X(X),F_Y(Y))$. Its empirical version is a $U$-statistic. If it is (unfortunately) asymmetric, this function is homogeneous of order 1, in the sense that, for $\lambda>0$, 
 $$GMD(\lambda X,\lambda Y)=\lambda GMD(X,Y).$$
And more generally, given a random vector $\boldsymbol{X}$ in $\mathbb{R}^d$, its Gini-covariance matrix is then
$$
 GMD(\boldsymbol{X}) = 2\mathbb{E}\big[\boldsymbol{X} G(\boldsymbol{X})^\top\big],
 $$
 which is a positive matrix (see \citet{dang2019gini,charpentier2021principal}). Thus, in dimension 2, we can define the Gini-covariance between vectors in $\mathbb{R}^d$ as,
$$
GMD(\x,\y) = \frac{2}{n^2} \sum_{i=1}^n x_i (2R_{\y}(y_i)-1).
$$
We can now use those dissimilarity measures to define a prametric distance.

\subsection{Gini Prametric}

Let $\X \equiv [x_{i,j}] \in \mathbb{R}^{n\times d}$ some data with $n$ points and $d$ features. The (cumulative) rank vector ${R}_{X_j}$ of feature $X_j$ assigns to each element of vector (column) $X_j$ its position (in ascending order of $X_j$). Inspired by the prametric defined as the  squared Euclidean L2 distance, 
$$
d_2(\x_i,\x_k)^2 = \sum\limits_{j=1}^d (x_{i,j} - x_{k,j})(x_{i,j} - x_{k,j}),
$$
we can define Gini prametric as follows:

\begin{definition}
 The \textit{Gini prametric} (or Gini distance) is defined for all $\x_i,\x_k \in \mathbb R^d$ such that $\x_i$ and $\x_k$ are two given rows of $\X$ (denoted ``$\x_i,\x_k \in \X$''):
\[
d_G(\x_i,\x_k)=\sum\limits_{j=1}^d (x_{i,j} - x_{k,j})({R}_{X_j}(x_{i,j}) - {R}_{X_j}(x_{k,j})),
\]
where ${R}_{X_j}(x_{i,j})$ denotes the ascending rank of element $x_{i,j}$ in $X_j=(x_{1,j},\cdots,x_{n,j})$.
\end{definition}

In practice, the average rank method is usually employed to deal with ties between many points, in this case the arithmetic mean of their ranks is computed. This avoids points with same values to be weighted by different ranks (this would imply for instance a bias in the Gini index, see \citet{yitzhaki2013gini}). 

This {Gini prametric} is the empirical version of the following function defined for all $\x,\y \in \mathbb R^d$, seen as realizations of a random vector $\X$, as (up to a multiplicative constant)
\[
d_G(\x,\y)=\sum\limits_{j=1}^d (x_{j} - y_{j})(F_j(x_{j}) - F_j(y_{j})),
\]
where $F_j$ is the marginal c.d.f. of the $j$-th component of random vector $\X$.

Back to the empirical version, contrary to L2 distance (or more generally any Minkowski $p$-distance with $p>1$) for which the distance is null if and only if $\x_i = \x_k$, many cases may arise for the Gini distance function. Let $R_X(\x_k) = (R_{X_1}(\x_{k,1}),\cdots, R_{X_d}(\x_{k,d}))$ be the rank vector of point $\x_k$, then
$$
d_G(\x_i,\x_k) =
\left \{
\begin{array}{ll}
   0,  & \text{if } \x_i = \x_k \\
   0,  & \text{if } R_X(\x_i) = R_X(\x_k) \\
   0,  & \text{if } x_{i,j} = x_{k,j} \text{ and/or } \\ &  R_{X_\ell}(x_{i,\ell}) = R_{X_\ell}(x_{k,\ell}) \\
   & \text{ for all } j,\ell \in \{1,\cdots,d\} 
\end{array}
\right.
$$
Consequently the Gini prametric is null if $\x_i = \x_k$, a standard property of distance functions. This condition does not imply $R_X(\x_i) = R_X(\x_k)$ since the position of the values of $\x_i$ (and $\x_k$)  depend on the rank within the features (columns) $X_j$. The second case $R_X(\x_i) = R_X(\x_k)$ implies $d_G(\x_i,\x_k)=0$. This means that the Gini prametric is sensitive to the rank dependency of the features. When all features have the same rank vectors, then the distance between each and every pairs of points is null. 

In general, the Gini prametric does not coincide with the Euclidean L2 distance nor the Manhattan L1 one, apart from specific cases. 
If a model based on $d_G$ needs to be trained again with a new observation $\x_{new}$, some rerankings of the points may occur implying new distances between the points. However, in some particular situations in which rank vectors are just translated by the same amount, the distance remains constant. With two points, $\x_1 = (0, 3) $ and $\x_2 = (4, 2) $,
$$
d_G(\x_1,\x_2) = (0-4)\cdot (1-2)+(3-2)\cdot (2-1)=4+1=5.
$$
But if there was a third point, $\x_{3} = (2,1.5)$, then the ranks of points $\x_1$ and $\x_2$ are rescaled, with respectively $R_X(\x_1) = (1,3)$ and $R_X(\x_2) = (3,2)$, so that
$$
d_G(\x_1,\x_2) = (0-4)\cdot (1-3)+(3-2)\cdot (3-2)=8+1=9.
$$
Accordingly, a strategy for computing \textit{conditional ranks} is necessary when new instances are added to the dataset (for train-test splits this will be described in Section \ref{section:clustering}). 
The main properties of the Gini prametric are the following. 

\begin{proposition}\label{proposition-distance} \emph{\textbf{(Gini prametric)}}
\newline Let $\x_i,\x_k \in \X$ such that $\X \in \mathbb R^{n\times d}$, let $\un^{n\times d}$ a matrix of ones, and let $\mathbb N$ be the set of positive integers:
\newline $\bulle$ $d_G(\x_i,\x_k) = 0$ if $\x_i=\x_k$ \textcolor{blue}{\emph{(Nullity)}}
\newline $\bulle$ $d_G(\x_i,\x_k) = 0$ if $R_X(\x_i)=R_X(\x_k)$ \textcolor{blue}{\emph{(Rank-Nullity)}}
\newline $\bulle$ $d_G(\x_i,\x_k) \geq 0 $ \textcolor{blue}{\emph{(Non-Negativity)}}
\newline $\bulle$ $d_G(\x_i,\x_k) = d_G(\x_k,\x_i) $ \textcolor{blue}{\emph{(Symmetry)}}
\newline $\bulle$ $d_G(\x'_i, \x'_k) = d_G(\x_i, \x_k)$ if $\X' = \X + \lambda \un^{n\times d}$ $\forall \lambda >0$
\textcolor{blue}{\emph{(Linear Invariance)}}
\newline $\bulle$ $d_G(\x'_i, \x'_k) = d_G(\x_i, \x_k)$ if $R_X(\x'_i) = R_X(\x_i) + \alpha \un^d$ and $R_X(\x'_k) = R_X(\x_k) + \alpha \un^d$ $\forall \alpha \in \mathbb N$ 
\textcolor{blue}{\emph{(Rank Invariance)}}
\end{proposition}

\begin{proof}
Straightforward from the properties of the Gini-covariance (\citet{yitzhaki2013gini}).    
\end{proof}

Linear invariance postulates that if the elements in $\X$ are increasing while their ranks remain constant, then the Gini prametric is invariant to such a translation. On the other hand, if a new point is included in $\X$, or if $\X$ is transformed in such a way that the ranks increase by the same amount, the Gini prametric remains invariant. Finally, it is non-negative, symmetric and does not satisfy the triangle inequality, therefore following the terminology employed by \citet{Arkhangelskii1990}, it can be referred to as a Gini \textit{symmetric}, such that $(\mathbb R^{n\times d}, d_G)$ is a Gini pramatric space.  

Notice that if $\overline{R}$ denotes descending ranks,
\[
d_G(\x_i,\x_k)=-\sum\limits_{j=1}^d (x_{i,j} - x_{k,j})(\overline{R}_{X_j}(x_{i,j}) - \overline{R}_{X_j}(x_{k,j})).
\]

\subsection{The Generalized Gini Prametric}\label{sect:Gini-cov}

In a series of papers \citet{Schechtman1987,schechtman2003family,yitzhaki2013gini} define a generalized version of the Gini Mean Difference:
$$
GMD_\nu(X,Y) = -{2\nu}\ \cov(X,\overline{F}_Y(Y)^{\nu-1}) ,  \ \nu > 1,
$$
with $\overline{F}_Y$ the decumulative distribution function of $Y$ (i.e., $\overline{F}_Y=1-F_Y$). \citet{yitzhaki2013gini} show that the Gini-covariance operator $\cov(X,\overline{F}_Y(Y)^{\nu-1})$ is a robust version of the standard covariance operator. The robustness to outliers  allows Gini-based measures to be enough flexible thanks to the hyper-parameter $\nu$. This flexibility yields robust estimators of regression parameters  as well as correlation coefficients \cite{Yitzhaki2013}. Additionally, in the case of $\nu = 2$ for  multivariate normal data, least squares regressions become a particular case of Gini regressions. Generalized Gini principal components analysis, as introduced by \citet{charpentier2021principal}, enables the influence of outliers to be minimized and helps find the best sub-space to fit the data according to the value of $\nu$ (for $\nu=2$ and multivariate normal data, the usual principal component analysis is recovered). Following \citet{schechtman2003family}, the use of decumulative ranks raised to the power of $\nu$ allows the Gini prametric to be rewritten as follows:

\begin{definition}
 The \textit{Generalized Gini prametric} $d_{G,\nu}(\x_i,\x_k)$ is defined for all $\x_i,\x_k \in \mathbb R^d$ such that $\x_i$ and $\x_k$ are two given rows of $\X$ (denoted ``$\x_i,\x_k \in \X$''):
\[
d_{G,\nu}(\x_i,\x_k)=-\sum\limits_{j=1}^d (x_{i,j} - x_{k,j})(\overline{R}_{X_j}(x_{i,j})^{\nu-1} - \overline{R}_{X_j}(x_{k,j})^{\nu-1}),
\]
for $\nu>1$, where $\overline{R}_{X_j}(x_{i,j})$ denotes the descending rank of element $x_{i,j}$ in $X_j=(x_{1,j},\cdots,x_{n,j})$.
\end{definition}

This generalized Gini prametric has the same property as $d_G$ except Rank Invariance due to the hyper-parameter $\nu$ (if $\nu\neq 2)$. 

\begin{proposition}\label{proposition-generalized-distance} \emph{\textbf{(Generalized Gini prametric)}}
\newline The generalized Gini prametric 
$d_{G,\nu}: \mathbb R^{d} \times \mathbb R^{d} \rightarrow \mathbb R_+$ 
satisfies Nullity (both), Non-Negativity, Symmetry and Linear Invariance.
\end{proposition}

\begin{proof}
Straightforward from the properties of the Gini-covariance (\citet{yitzhaki2013gini}). 
\end{proof}

The generalized version of $d_G$ offers more flexibility about the rank correlation between the features to be taken into account. As in the Gini-covariance, $\nu=2$ is referred to be the median value, which places as much weight on the lower part of the feature distributions as on the upper part. For $\nu > 2$, more weights are put to higher ranks, whereas for $\nu < 2$, more weights are put on lower ranks. 


In order to interpret the correlations that lie behind the generalized Gini prametric, let us take centered variables $\x_k^c$ and centered rank vectors $\overline{R}_X^c(\x_k)^{\nu-1}$, then:
\begin{align}
& d_{G,\nu}^c(\x_i^c,\x_k^c) := \notag \\
& \ n\cov(\x_i^c,\overline{R}_X^c(\x_k)^{\nu-1})+n\cov(\x_k^c,\overline{R}_X^c(\x_i)^{\nu-1})  \notag \\
&-n\cov(\x_i^c,\overline{R}_X^c(\x_i)^{\nu-1})-n\cov(\x_k^c,\overline{R}_X^c(\x_k)^{\nu-1}). \notag
\end{align}
The two first terms of the previous relation correspond to the two Gini-covariances (up to $n$): $\x_i$ with the rank of $\x_k$ and conversely $\x_k$ with the ranks of $\x_i$. The two last terms are (up to $n$) generalized Gini Mean Difference, which correspond to the intrinsic variability of $\x_i$ and $\x_k$ respectively. Everything happens as if the generalized Gini prametric measured on points belonging to different groups captures the two Gini-covariances between groups in excess of the covariances within groups. In this respect, it could be included in the family of distances (F3) based on scalar products.

\section{Gini KNN and Gini K-means}\label{section:clustering}

Algorithms implemented in a Gini prametric space $(\mathbb{R}^{n\times d},d_{G,\nu})$ may be used instead of those relying on standard L$p$ metric space $(\mathbb{R}^d,d_p)$ for a given integer $p \geq 1$. We will discuss two algorithms in this section, Gini KNN (\ref{sec:KNN}) and Gini K-means (\ref{sec:kmeans}).

\subsection{Gini KNN}\label{sec:KNN}

K-Nearest Neighbors (KNN) is a non-parametric algorithm used for classification or regression. After identifying the $k$ closest data points (neighbors) to a given input based on a distance metric (e.g. Gini prametric, Euclidean, Manhattan, etc.), the KNN consists in predicting the output based on the majority class (classification) or average value (regression) of these neighbors. 

The Gini KNN relies on two hyper-parameters to learn during the training phase of the KNN: the number of neighbors $k$ and the value of the hyper-parameter $\nu$. Since the approach is supervised there is no need to define \textit{ex ante} the $\nu$ value, which can be deduced to be the best parameter that maximizes the F-measure or other pre-defined metrics (precision, recall, etc.). In the same time, the optimal value of $k$ may be deduced from the maximization of the same metrics. The prediction of the test points (regression task or classification task) depends on the computation of their \textit{conditional ranks}, necessary for the computation of the generalized Gini prametric. The conditional ranks are ranks the test points would have if they were part of the training dataset. These are the $n_{te}$ last elements extracted from vector $\overline{R}_{\X}(\X)^\nu (n_{te}/n)$, with $n_{te}$ and $n$ the size of the testing data and the size of the whole dataset respectively.

Finally, the algorithm is convergent and its error rate is bounded.  

\begin{proposition}\label{KNN_convergence}\emph{\textbf{(KNN convergence and error rate bound)}}
\newline $\bulle$ The KNN algorithm converges using the generalized Gini prametric $d_{G,\nu}$, for all $\nu \neq 1$. 
\newline $\bulle$ Let $M$ be the number of classes, $R^{*}$ the Bayes error rate and $R_{KNN}$ the asymptotic KNN error rate, then
$$R^{*} \leq R_{KNN} \leq R^{*}\Big(2-\displaystyle\frac{MR^{*}}{M-1}\Big).
$$ 
\end{proposition}

\begin{proof}
The proof is exactly the same as the one of \citet{coverhart1967}. The convergence is true for any type of metric used and the bound on the error rate does not change.
\end{proof}

\subsection{Gini K-means}\label{sec:kmeans}

K-means is an unsupervised clustering algorithm that groups data points into $k$ clusters by minimizing the Euclidean distance (variance), or other distance metrics, within each cluster. It iteratively assigns each point to the nearest cluster centroid and updates the centroids by calculating the mean of the points within each cluster. The process is repeated until the centroids converge or a maximum number of iterations is reached.

The Gini K-means algorithm based on the generalized Gini prametric may be implemented if two conditions are met. First, convergence must be ensured for any given dataset and hyper-parameters $\nu$. 

\begin{proposition}\label{proposition-K-means}\emph{\textbf{(Convergence)}}
\newline The Gini K-means algorithm based on the generalized Gini prametric $d_{G,\nu}$ is convergent whenever rank vectors stay constant, for all $\nu \neq 1$.  
\end{proposition}

\begin{proof}
See the Supplementary Materials \citep{mussard2025}.
\end{proof}

Second, since convergence is always guaranteed, an optimal $\nu$ value may be derived. Setting training and testing data as $\X := [\X_{tr}^\top, \X_{te}^\top]^\top$, a grid search may be conducted as follows. The Gini K-means algorithm is trained on $\X_{tr}$ and the silhouette criterion \citep{rousseeuw1987silhouettes} is computed for each value of $\nu$ (not exceeding 6, otherwise too much weight is placed on minimal values of $\X$). The optimal hyper-parameter $\nu^*$ is the value that maximizes the silhouette score. Once $\nu^*$ is determined, the labels of the testing data can be inferred. 

\begin{algorithm}[H]
\caption{
\begin{center}
\textbf{Gini K-means: $\nu^*$}
\end{center}
}\label{MC}
\begin{algorithmic}
\STATE {\bfseries Input:} training data $\X_{tr}$ 
\FOR{$\nu$ in $[0.1 ; 6]$ $\nu\neq1$}
\FOR{fold in folds}
\STATE $K$ random centroids in $\X_{tr}$[fold] (\texttt{k-means++}) \; 
\REPEAT 
\STATE labels $\leftarrow$ $\min d_{G,\nu}(\X_{tr}$[fold],$centroids)$ 
\STATE update centroids
\UNTIL $\Delta d_{G,\nu}(\X_{tr}$[fold],$centroids)=\mathbf{0}$
\STATE $\overline{\boldsymbol{R}}_{\X}=$ conditional ranks of $\X_{tr}$[$-$fold]
\STATE labels $\leftarrow$ $\min d_{G,\nu}(\X_{tr}$[$-$fold],$centroids;\overline{\boldsymbol{R}}_{\X})$
\STATE silh\_score[fold] = silhouette($\X_{tr}$[fold], labels)
\ENDFOR
\STATE mean\_silh\_score$(\nu)$ = mean(silh\_score)
\ENDFOR
\STATE {\bf return} $\nu^* = \arg\max$ mean\_silh\_score$(\nu)$
\end{algorithmic}
\end{algorithm}

Once the hyper-parameter is obtained, Euclidean and Gini K-means may be compared, as depicted in Figure \ref{k-means-blobs}.

\begin{figure}[h!]
    \centering
\includegraphics[width=0.9\linewidth]{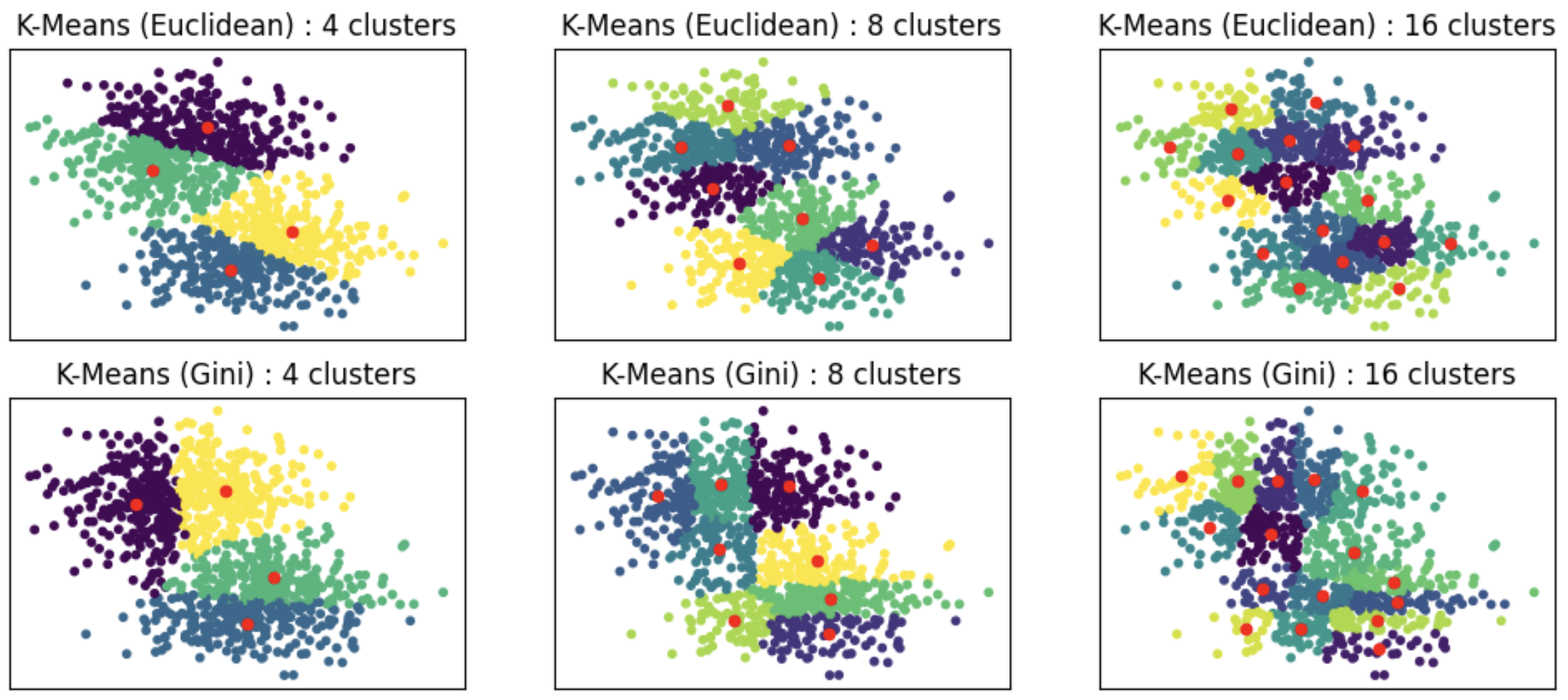}
    \caption{Example: Euclidean vs. Generalized Gini prametric $\nu^*=3.52$}\label{k-means-blobs}
\end{figure}
\bigskip
\bigskip

\begin{figure}[!ht]
    \centering
\includegraphics[width=0.7\linewidth]{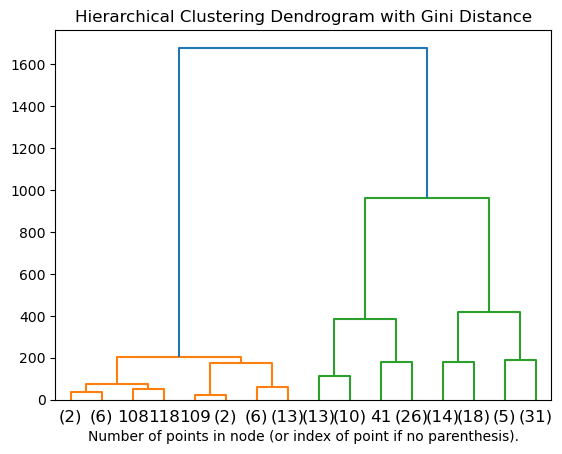}
\caption{Example of agglomerative clustering: $\nu=2$}
\label{fig:HAC}
\end{figure}
\bigskip \ 
\bigskip

Finally, other clustering techniques relying on the Euclidean distance, close to K-means, may be adapted to the generalized Gini prametric, such as the agglomerative clustering (see subsection \ref{sect:agglo}). An illustration is depicted in Figure \ref{fig:HAC} on the iris dataset.\footnote{ \href{https://github.com/giniknnkmeans/KNN_Kmeans_Gini_prametric}{See Github}.}

\section{Experiments}\label{exp}
K-means and KNN classifications have been performed adapting the \texttt{scikit-learn} library to the generalized Gini prametric. An additional illustration is also conducted with agglomerative clustering. These algorithms are compared on the basis of 12 distances issued from the 8 families (described in Table \ref{tab:distances}): L$p$ Minkowski distance (Euclidean, Minkowski, Manhattan), L1 distance (Canberra), Inner product distance (Cosine), Squared Chord distance (Hellinger distance), Squared L2 distance (Pearson $\chi^2$), Shannon entropy distance (Jensen-Shannon), Vicissitude distance (Vicis Symmetric), and other distances (Hassanat). The comparison is performed over 16 datasets of UCI repository\footnote{\href{https://archive.ics.uci.edu/}{UCI datasets repository link.}}, without normalization in order to observe the impact of outliers (if present) on the performance. Following the study of \citet{Prasatha2019}, the 16 datasets have been chosen with regard to the excellent performance of some distances robust to outliers, in particular the Hassanat distance. The datasets are described in Table~\ref{datasets}. The detailed metrics of the Gini K-means and Gini KNN experiments, analyzed in the following subsections, are relegated in Supplementary Materials \citep{mussard2025}. 

\begin{table}[ht]
\centering
\resizebox{\columnwidth}{!}{
\begin{tabular}{||l |c |c |c| c| c| c ||} 
 \hline
 Name & \#E & \#F & \#C & Data type & Min & Max  \\ [0.5ex] 
 \hline\hline
 \textbf{(A)} Australian & 690 & 14 & 2 & Real & 0 & 100001  \\
 \hline
\textbf{(B)} Banknote & 1372 & 4 & 2 & Real & -13.77 & 17.93  \\
 \hline
\textbf{(Bc)} Breast-Cancer & 569 & 30 & 2 &Real, numeric & 1 & 13454352  \\
 \hline
\textbf{(I)} Ionosphere & 351 & 34 & 2 & Real & -1 & 1 \\
 \hline
\textbf{(Ir)} Iris & 150 & 4 & 3 & Real, numeric & 0.1 & 7.9 \\ 
 \hline
\textbf{(G)} German & 1000 & 20 & 2 & Integer, Categorical & 0 & 184  \\
 \hline
\textbf{(Gl)} Glass & 214 & 9 & 6 & Real & 0 & 75.41  \\
 \hline
\textbf{(He)} Heart & 303 & 13 & 2 & Integer & 0 & 564  \\
 \hline
\textbf{(QS)} QSAR & 1060 & 41 & 2 & Real & -5.256 & 147  \\
 \hline
\textbf{(S)} Sonar & 208 & 60 & 2 & Real, Categorical & 0 & 1  \\ 
\hline
\textbf{(V)} Vehicle & 946 & 18 & 4 & Integer & 0 & 1018 \\
 \hline
\textbf{(Wi)} Wine & 178 &13& 3 & Real, numeric& 813 &1680 \\
  \hline
\textbf{(Ba)} Balance & 625& 4& 3 & Positive Integer&1 &5 \\
  \hline
\textbf{(Ha)} Haberman & 306 & 3& 2 & Positive Integer &0 &83 \\
  \hline
\textbf{(W)} Wholesale & 440 & 8& 2 &Positive Integer &1 & 112151\\
  \hline
\textbf{(In)} Indian Liver & 583 & 10 & 2 & Real, Integer & 0 & 297 \\
 \hline
\end{tabular}
}
\caption{Description of UCI datasets used where \#E = number of examples, \#F = number of features, \#C = number of classes}
\label{datasets}
\end{table}

\subsection{KNN Experiments}

The generalized Gini prametric  is implemented with the supervised KNN algorithm using the \texttt{scikit-learn} library. To compare the 12 distances listed above, precision and recall metrics are used on the same UCI datasets described in Table \ref{datasets}. To identify the optimal number of neighbors, $k$ is looped from $1$ to $11$ for each distance. A grid search is then performed to select the optimal hyper-parameter $\nu$ of the generalized Gini prametric that maximizes the precision. Then, recall and precision are computed and a 3-fold cross-validation is performed to validate the results (see Tables 1-10 in Supplementary Materials for experiments with and without noise \citep{mussard2025}).

\medskip

\textbf{Without Noise.} All datasets are treated as a task of classification (for regressions the output is discretized in order to compute the confusion matrix). Table \ref{tab:rankings_precision_knn} provides the precision rankings of the 12 KNN models evaluated across the 16 UCI datasets. The Gini KNN achieves the highest rank on 8 datasets, with an average rank of 3.12 (last column). The second-best KNN model, based on Hassanat distance, has an average rank of 4.06, with top positions on 5 datasets. These KNN models are followed by those relying on Hellinger distance (F4 family) and Gini ($\nu=2$), with 4.62 and 5.12 average ranks, respectively.  

\begin{table}[ht]
\centering
\tiny
\renewcommand{\arraystretch}{1} 
\setlength{\tabcolsep}{2pt} 
\begin{tabular}{|p{0.8in}|c|c|c|c|c|c|c|c|c|c|c|c|c|c|c|c|c|}
\hline
\textbf{Distances $\downarrow$}          & \textbf{Ir} & \textbf{Wi} & \textbf{Bc} & \textbf{S} & \textbf{QS} & \textbf{V} & \textbf{A} & \textbf{Gl} & \textbf{I} & \textbf{B} & \textbf{G} & \textbf{He} & \textbf{Ba} & \textbf{Ha} & \textbf{W} & \textbf{In} & \textbf{Rank} \\ \hline
Gini prametric $\nu^*$ & 1 & 6 & 4 & 1 & 1 & 1 & 4 & 1 & 4 & 1 & 1 & 4 & 3 & 10 & 1 & 7 & \cellcolor{red!40}3.12\\
Gini prametric $\nu=2$ & 4 & 8 & 4 & 2 & 2 & 1 & 5 & 5 & 9 & 1 & 4 & 8 & 7 & 12 & 3 & 7 & \cellcolor{blue!10} 5.12 \\
Euclidean & 4 & 11 & 9 & 10 & 5 & 6 & 11 & 3 & 6 & 1 & 11 & 9 & 10 & 2 & 6 & 2 & 6.62 \\
Manhattan & 10 & 10 & 8 & 6 & 5 & 3 & 9 & 2 & 3 & 1 & 6 & 9 & 6 & 3 & 6 & 2 & 5.56 \\
Minkowski & 4 & 12 & 9 & 10 & 9 & 9 & 11 & 8 & 8 & 1 & 11 & 9 & 7 & 1 & 10 & 6 & 7.81 \\
Cosine & 1 & 9 & 9 & 6 & 9 & 4 & 10 & 6 & 6 & 7 & 10 & 12 & 2 & 3 & 12 & 12 & 7.38 \\
Hassanat & 11 & 3 & 1 & 2 & 9 & 4 & 2 & 3 & 1 & 1 & 1 & 1 & 10 & 3 & 6 & 7 & \cellcolor{blue!30}4.06 \\
Canberra & 11 & 2 & 1 & 9 & 3 & 9 & 1 & 10 & 1 & 10 & 6 & 2 & 10 & 8 & 3 & 7 & 5.81 \\
Hellinger & 4 & 6 & 4 & 2 & 5 & 7 & 5 & 6 & 4 & 8 & 6 & 3 & 3 & 8 & 1 & 2 & \cellcolor{blue!20}4.62 \\
Jensen-Shannon & 1 & 4 & 9 & 2 & 5 & 7 & 5 & 10 & 11 & 12 & 4 & 4 & 1 & 3 & 11 & 1 & 5.62 \\
Pearson Chi2 & 4 & 5 & 4 & 6 & 12 & 12 & 8 & 8 & 12 & 8 & 1 & 4 & 3 & 10 & 3 & 7 & 6.69 \\
Vicis Symmetric & 4 & 1 & 1 & 12 & 3 & 11 & 3 & 12 & 10 & 11 & 6 & 7 & 7 & 3 & 6 & 2 & 6.19 \\
\hline
\end{tabular}
\caption{Ranking of KNN Models by Precision}
\label{tab:rankings_precision_knn}
\end{table}
\bigskip

The same results are obtained in terms of recall (Table \ref{tab:rankings_recall_knn}). The Gini KNN achieves the top rank 9 times, while Hassanat achieves it 5 times. The standard Gini KNN ($\nu=2$) ranks third.
\bigskip

\begin{table}[ht]
\centering
\tiny
\renewcommand{\arraystretch}{1} 
\setlength{\tabcolsep}{2pt} 
\begin{tabular}{|p{0.8in}|c|c|c|c|c|c|c|c|c|c|c|c|c|c|c|c|c|}
\hline
\textbf{Distances $\downarrow$}          & \textbf{Ir} & \textbf{Wi} & \textbf{Bc} & \textbf{S} & \textbf{QS} & \textbf{V} & \textbf{A} & \textbf{Gl} & \textbf{I} & \textbf{B} & \textbf{G} & \textbf{He} & \textbf{Ba} & \textbf{Ha} & \textbf{W} & \textbf{In} & \textbf{Rank} \\ \hline
Gini prametric $\nu^*$& 1 & 7 & 3 & 1 & 1 & 1 & 3 & 1 & 5 & 1 & 1 & 7 & 3 & 6 & 1 & 1 & \cellcolor{red!40}2.69 \\
Gini prametric $\nu=2$ & 2 & 8 & 3 & 2 & 1 & 1 & 7 & 3 & 9 & 1 & 4 & 8 & 3 & 12 & 3 & 1 & \cellcolor{blue!20}4.25 \\
Euclidean & 2 & 11 & 9 & 9 & 3 & 6 & 11 & 2 & 5 & 1 & 11 & 10 & 6 & 6 & 7 & 1 & 6.25 \\
Manhattan & 8 & 10 & 8 & 4 & 3 & 3 & 9 & 3 & 3 & 1 & 10 & 9 & 10 & 1 & 5 & 1 & 5.5 \\
Minkowski & 2 & 12 & 9 & 9 & 9 & 9 & 12 & 11 & 8 & 1 & 11 & 12 & 6 & 1 & 10 & 6 & 8 \\
Cosine & 2 & 9 & 12 & 8 & 11 & 5 & 10 & 5 & 5 & 7 & 8 & 11 & 2 & 1 & 11 & 6 & 7.06 \\
Hassanat & 10 & 3 & 1 & 2 & 9 & 4 & 2 & 7 & 1 & 1 & 2 & 3 & 11 & 1 & 7 & 1 & \cellcolor{blue!30}4.06 \\
Canberra & 10 & 2 & 3 & 9 & 3 & 9 & 1 & 7 & 2 & 10 & 4 & 1 & 12 & 6 & 7 & 6 & 5.75 \\
Hellinger & 2 & 5 & 3 & 4 & 3 & 7 & 6 & 5 & 3 & 8 & 8 & 2 & 6 & 1 & 1 & 6 & \cellcolor{blue!10}4.38 \\
Jensen-Shannon & 8 & 4 & 9 & 4 & 3 & 7 & 5 & 10 & 12 & 12 & 4 & 3 & 1 & 6 & 11 & 12 & 6.94 \\
Pearson Chi2 & 2 & 5 & 3 & 4 & 12 & 12 & 7 & 9 & 11 & 8 & 2 & 3 & 3 & 6 & 3 & 6 & 6 \\
Vicis Symmetric & 10 & 1 & 1 & 12 & 3 & 11 & 3 & 12 & 10 & 11 & 4 & 3 & 6 & 11 & 5 & 11 & 7.12 \\
\hline
\end{tabular}
\caption{Ranking of KNN Models by Recall}
\label{tab:rankings_recall_knn}
\end{table}
\bigskip

\textbf{Wilcoxon test.} To better highlight the performance differences between the KNN methods, Table \ref{tab:wilcoxon_pvalues} shows the Wilcoxon signed-rank test results, assessing whether there are statistically significant differences in recall and precision distributions across the 16 datasets. No statistical difference is observed between the generalized Gini parametric (or Gini) and Hassanat distance. However, a significant statistical difference is found between the generalized Gini parametric and cosine (and at the 10\% significance level for Gini vs. cosine).
\begin{table}[h!]
\centering
\begin{tabular}{lcc}
\hline
Comparison & Precision & Recall \\
\hline
Gini prametric vs. Hassanat             & 0.1944 & 0.4842 \\
Gini prametric vs. Cosine               & \textbf{0.0934} & \textbf{0.0822} \\
Generalized Gini prametric vs. Hassanat & 0.9441 & 0.5393 \\
Generalized Gini prametric vs. Cosine   & \textbf{0.0113} & \textbf{0.0020} \\
\hline
\end{tabular}
\caption{Wilcoxon p-values for Precision and Recall}
\label{tab:wilcoxon_pvalues}
\end{table}

\textbf{With 5\% and 10\% Noise.} 
The experiment involves 5\% and 10\% noise into the data. Specifically, Gaussian noise with a mean of 0 and a standard deviation of 1 is added. This setup is designed to evaluate the robustness of the KNN algorithms under less favorable conditions, where the generalized Gini prametric is expected to outperform the other methods.
In the first case, the Gini KNN achieves the highest average precision (11 top-ranked positions out of 16) and recall (10 top-ranked positions). In the 10\% contamination case, the performances are much more variable. The generalized Gini prametric achieves only 6 best positions for precision and 5 for recall. Then, it comes Hassanat, Minkowski and Cosine (2 top positions for precision), and Hassanat again for recall (3 top positions, see Tables 3-10 in Supplementary Materials \citep{mussard2025}). 

\medskip

\textbf{Time complexity.} The time complexity of the generalized Gini prametric is $\mathcal{O}(nd \log n)$ for a matrix $\X$ of size $n \times d$. To compare the running time of Gini KNN and standard KNN (Minkowski with $p=2$), the Fashion MNIST dataset is used \citep{fashion-mnist}. The training set consists of 49{,}000 images and the test set consists of 21{,}000 images, each of size $28 \times 28$. The Gini KNN ($\nu=2$) achieves a macro F-measure nearly identical to the Minkowski KNN (0.86 vs.\ 0.85 respectively). The running time is approximately 3.5 minutes on an RTX 8000 GPU for both methods. Both distances are precomputed (using the KNN implementation from the \texttt{sklearn} library), and only a single loop is used in the computation of each distance for a fair comparison.

\subsection{K-means Experiments}\label{k_means_exp}

To ensure accurate performance measurement, a 5-fold cross-validation is performed with a fixed seed for each experiment. The number of clusters is set to $k$ i.e. the number of classes in Table~\ref{datasets}. K-means are first initialized with the \texttt{k-means++} algorithm (\citet{arthur2007kmeans}), based on a sampling of the empirical probability distributions of the points (of the entire dataset) that accelerates convergence. All K-means models (issued from the 12 distances) are then initialized with the same optimized centroids. 
For Gini K-means, a grid search is performed over the hyper-parameter $\nu$, selecting the value that maximizes the precision score. Indeed, since the UCI datasets are supervised, computing the silhouette criterion is not required.
After predicting the labels on the test data, the Kuhn-Munkres algorithm \cite{kuhn1955hungarian} is applied to align the predicted clusters with the true labels, as K-means does not assign labels corresponding to the actual classes. Finally, the average precision and recall are computed across all folds. 

\medskip

\textbf{Without Noise.} The generalized Gini prametric achieves the best performance thanks to its hyper-parameter. It is top-ranked on 5 datasets out of 16 (in terms of precision Table \ref{tab:rankings_precision_kmeans}), with a mean rank of 3.5 (the best mean rank of 4.19 for recall Table \ref{tab:rankings_recall_kmeans}). The second best method is the K-means based on the Hassanat distance in terms  of precision and recall (mean rank of 5.06 and 4.75 respectively). The details about precision and recall are relegated in Tables 11-14 of Supplementary Materials \citep{mussard2025}. 

\begin{table}[h!]
\centering
\tiny
\renewcommand{\arraystretch}{1} 
\setlength{\tabcolsep}{2pt} 
\begin{tabular}{|p{0.8in}|c|c|c|c|c|c|c|c|c|c|c|c|c|c|c|c|c|}
\hline
\textbf{Distances $\downarrow$}   &  \textbf{I} & \textbf{B} & \textbf{G} & \textbf{He} & \textbf{Ba} & \textbf{Ha} & \textbf{W} & \textbf{In} &  \textbf{Ir} & \textbf{Wi} & \textbf{Bc} & \textbf{S} & \textbf{QS} & \textbf{V} & \textbf{A} & \textbf{Gl} &  \textbf{Rank} \\ \hline
Gini prametric $\nu^*$ &          4 &        1 &      1 &     1 &       1 &        5 &         6 &                    3 &    4 &    6 &             3 &     1 &    8 &       6 &          2 &     4 &         \cellcolor{red!40}3.50 \\
Gini prametric $\nu=2$            &          7 &        8 &      4 &     5 &       4 &        7 &        12 &                    6 &    5 &    8 &             4 &     2 &   12 &       7 &          3 &     8 &         6.38 \\
Euclidean        &         10 &        4 &      5 &     8 &       5 &       11 &         3 &                   10 &    8 &   11 &             8 &     7 &   10 &      10 &         10 &    10 &         8.12 \\
Manhattan        &          8 &        3 &      5 &    10 &       3 &        9 &         1 &                   10 &    6 &    9 &             9 &     5 &    9 &      11 &         10 &     6 &         7.12 \\
Minkowski        &          9 &        2 &      5 &     7 &       6 &       10 &         5 &                   10 &    7 &   10 &             7 &     6 &   11 &       9 &         10 &    11 &         7.81 \\
Cosine           &          6 &        6 &     10 &     3 &       2 &        8 &         2 &                    8 &   12 &    7 &             5 &     8 &    6 &      12 &          5 &     7 &         6.69 \\
Canberra         &          2 &        7 &      3 &     2 &       9 &        1 &        11 &                    5 &   11 &    2 &             2 &    10 &    4 &       8 &          9 &     5 &         5.69 \\
Hellinger        &          1 &        5 &      8 &     9 &       8 &        6 &         9 &                    1 &    2 &    4 &            11 &    12 &    1 &       2 &          7 &     2 &         \cellcolor{blue!20}5.5 \\
Jensen-Shannon   &         11 &       12 &      9 &    12 &      10 &       12 &        10 &                    1 &    1 &   12 &            12 &     9 &    2 &       3 &          4 &    12 &         8.25 \\
Pearson Chi2      &         11 &       11 &     12 &     6 &      11 &        3 &         4 &                    7 &    3 &    1 &             6 &     3 &    5 &       4 &          6 &     1 &         \cellcolor{blue!10}5.88 \\
Vicis Symmetric  &          3 &       10 &     11 &    11 &      12 &        2 &         8 &                    9 &    9 &    5 &            10 &    11 &    3 &       1 &          8 &     3 &         7.25 \\
Hassanat         &          5 &        9 &      2 &     4 &       7 &        4 &         7 &                    4 &   10 &    2 &             1 &     4 &    7 &       5 &          1 &     9 &         \cellcolor{blue!30}5.06 \\
\hline
\end{tabular}
\caption{Ranking of K-means Models by Precision}
\label{tab:rankings_precision_kmeans}
\end{table}
\bigskip

\begin{table}[htbp!]
\centering
\tiny
\renewcommand{\arraystretch}{1} 
\setlength{\tabcolsep}{2pt} 
\begin{tabular}{|p{0.8in}|c|c|c|c|c|c|c|c|c|c|c|c|c|c|c|c|c|}
\hline
\textbf{Distances $\downarrow$}   &  \textbf{I} & \textbf{B} & \textbf{G} & \textbf{He} & \textbf{Ba} & \textbf{Ha} & \textbf{W} & \textbf{In} &  \textbf{Ir} & \textbf{Wi} & \textbf{Bc} & \textbf{S} & \textbf{QS} & \textbf{V} & \textbf{A} & \textbf{Gl} &  \textbf{Rank} \\ \hline
Gini prametric $\nu^*$ &          4 &        1 &      2 &     2 &       2 &        9 &         7 &                    7 &    4 &    6 &             2 &     1 &    6 &       6 &          4 &     4 &         \cellcolor{red!40}4.19 \\
Gini prametric $\nu=2$             &          9 &        8 &      3 &     3 &       3 &        7 &         7 &                    6 &    4 &    9 &             4 &     3 &    8 &       7 &          5 &     6 &         5.75 \\
Euclidean        &         12 &        4 &      5 &     5 &       4 &       10 &        10 &                   10 &    7 &   11 &             8 &     8 &   11 &       9 &          9 &    10 &         8.31 \\
Manhattan        &         10 &        2 &      5 &     8 &       1 &        8 &        12 &                   10 &    4 &   10 &             9 &     5 &    9 &      10 &          9 &     8 &         7.5 \\
Minkowski        &         11 &        2 &      5 &     7 &       6 &       11 &         6 &                   10 &    7 &   11 &             7 &     7 &   12 &       8 &          9 &    10 &         8.06 \\
Cosine           &          6 &        5 &     11 &     4 &       5 &        6 &         2 &                    4 &   12 &    8 &             5 &     9 &    4 &      12 &          2 &     7 &         6.38 \\
Canberra         &          3 &        6 &      4 &     9 &      10 &        3 &        11 &                    2 &   10 &    3 &             1 &    11 &    2 &      10 &         12 &     5 &         6.38 \\
Hellinger        &          2 &        7 &     12 &    10 &       9 &        4 &         4 &                    9 &    2 &    4 &            11 &    12 &    3 &       4 &          7 &     1 &         \cellcolor{blue!10}6.31 \\
Jensen-Shannon   &          7 &        9 &      8 &    12 &      12 &       12 &         5 &                    8 &    1 &    7 &            12 &     4 &    1 &       1 &          3 &    12 &         7.12 \\
Pearson Chi2      &          7 &       12 &     10 &     6 &      11 &        2 &         1 &                    3 &    3 &    1 &             6 &     2 &    5 &       3 &          6 &     2 &         \cellcolor{blue!20}5 \\
Vicis Symmetric  &          1 &       11 &      9 &    11 &       8 &        5 &         9 &                    1 &   11 &    5 &            10 &    10 &    7 &       2 &          8 &     3 &         6.94 \\
Hassanat         &          5 &       10 &      1 &     1 &       7 &        1 &         3 &                    5 &    9 &    2 &             2 &     6 &   10 &       5 &          1 &     8 &         \cellcolor{blue!30}4.75 \\\hline
\end{tabular}
\caption{Ranking of K-means Models by Recall}
\label{tab:rankings_recall_kmeans}
\end{table}

\textbf{With 5\% and 10\% Noise.}
The same experiments are conducted for different levels of noise (5\% and 10\%) as in the KNN experiments, under the same conditions. 
The Gini K-means reaches the top position 3 times (out of 16 datasets) with a mean ranking of 3.75 (5\% noise). The Hellinger distance is ranked 2nd, with more top-ranked positions (5 out of 16 datasets), but with a larger variability so that the mean rank is only 5 (like Hassanat). In the case of 10\% noise, the generalized Gini prametric has only 2 top-ranked positions and a mean rank of 4.19. Again, Hellinger has much more top-ranked positions (5 out of 16) but a mean rank of 5.06, exhibiting a more important variability in the ranking (sometimes ranked 10 or 11). The results are presented in Supplementary Materials (Tables 15-22 \citep{mussard2025}).

\medskip

\textbf{Strategy to select $\nu$ for real applications.} The performances described above have been computed with known labels for the test points. In real applications, once $k$ is fixed, a strategy is necessary to select the hyper-parameter. 
As described in Algorithm \ref{MC}, a 5-fold cross-validation with a fixed seed is used to select the hyper-parameter $\nu^*$ that maximizes the silhouette score (mean across all folds). Tables \ref{tab:rankings_precision_silhouette} and \ref{tab:rankings_recall_silhouette} indicate that the generalized Gini prametric is ranked 4th in precision and 3rd in recall (for $\nu=2$). This experiment shows that maximizing the silhouette score is not the best method, but it allows the generalized Gini prametric to be ranked among the top 4 distances (Tables 23-26 in Supplementary Materials \citep{mussard2025}). 

\begin{table}[htbp!]
\centering
\tiny
\renewcommand{\arraystretch}{1} 
\setlength{\tabcolsep}{2pt} 
\begin{tabular}{|p{0.8in}|c|c|c|c|c|c|c|c|c|c|c|c|c|c|c|c|c|}
\hline
\textbf{Distances}   &  \textbf{I} & \textbf{B} & \textbf{G} & \textbf{He} & \textbf{Ba} & \textbf{Ha} & \textbf{W} & \textbf{In} &  \textbf{Ir} & \textbf{Wi} & \textbf{Bc} & \textbf{S} & \textbf{QS} & \textbf{V} & \textbf{A} & \textbf{Gl} &  \textbf{Rank} \\ \hline
Gini prametric $\nu^*$ &          8 &        1 &      4 &     2 &       6 &        6 &         9 &                    3 &    5 &    8 &             6 &     6 &   11 &       8 &          2 &     5 &         \cellcolor{blue!10}5.62 \\
Gini prametric $\nu=2$               &          6 &        8 &      3 &     5 &       3 &        7 &        12 &                    6 &    4 &    7 &             3 &     1 &   12 &       6 &          3 &     8 &         5.88 \\
Euclidean          &         10 &        4 &      5 &     8 &       4 &       11 &         3 &                   10 &    8 &   11 &             8 &     7 &    9 &      10 &         10 &    10 &         8 \\
Manhattan          &          7 &        3 &      5 &    10 &       2 &        9 &         1 &                   10 &    6 &    9 &             9 &     4 &    8 &      11 &         10 &     6 &         6.88 \\
Minkowski          &          9 &        2 &      5 &     7 &       5 &       10 &         5 &                   10 &    7 &   10 &             7 &     5 &   10 &       9 &         10 &    11 &         7.62 \\
Cosine             &          5 &        6 &     10 &     3 &       1 &        8 &         2 &                    8 &   12 &    6 &             4 &     8 &    6 &      12 &          5 &     7 &         6.44 \\
Canberra           &          2 &        7 &      2 &     1 &       9 &        1 &        11 &                    5 &   11 &    2 &             2 &    10 &    4 &       7 &          9 &     4 &          \cellcolor{blue!20}5.44 \\
Hellinger          &          1 &        5 &      8 &     9 &       8 &        5 &         8 &                    1 &    2 &    4 &            11 &    12 &    1 &       2 &          7 &     2 &         \cellcolor{blue!30}5.38 \\
Jensen-Shannon     &         11 &       12 &      9 &    12 &      10 &       12 &        10 &                    1 &    1 &   12 &            12 &     9 &    2 &       3 &          4 &    12 &         8.25 \\
PearsonChi2        &         11 &       11 &     12 &     6 &      11 &        3 &         4 &                    7 &    3 &    1 &             5 &     2 &    5 &       4 &          6 &     1 &         5.75 \\
VicisSymmetric1    &          3 &       10 &     11 &    11 &      12 &        2 &         7 &                    9 &    9 &    5 &            10 &    11 &    3 &       1 &          8 &     3 &         7.19 \\
Hassanat           &          4 &        9 &      1 &     4 &       7 &        4 &         6 &                    4 &   10 &    2 &             1 &     3 &    7 &       5 &          1 &     9 &         \cellcolor{red!40}4.81 \\\hline
\end{tabular}

\caption{Ranking of K-means Models by Precision with Silhouette}
\label{tab:rankings_precision_silhouette}
\end{table}

\bigskip

\begin{table}[htbp!]
\centering
\tiny
\renewcommand{\arraystretch}{1} 
\setlength{\tabcolsep}{2pt} 
\begin{tabular}{|p{0.8in}|c|c|c|c|c|c|c|c|c|c|c|c|c|c|c|c|c|}
\hline
\textbf{Distances}   &  \textbf{I} & \textbf{B} & \textbf{G} & \textbf{He} & \textbf{Ba} & \textbf{Ha} & \textbf{W} & \textbf{In} &  \textbf{Ir} & \textbf{Wi} & \textbf{Bc} & \textbf{S} & \textbf{QS} & \textbf{V} & \textbf{A} & \textbf{Gl} &  \textbf{Rank} \\ \hline
Gini prametric $\nu^*$ &         10 &        1 &      2 &     3 &       3 &        8 &        12 &                    8 &    4 &    9 &             6 &     6 &    7 &      11 &          6 &     5 &         6.31 \\
Gini prametric $\nu=2$              &          8 &        8 &      3 &     2 &       2 &        7 &         7 &                    6 &    4 &    8 &             3 &     2 &    8 &       6 &          4 &     6 &         \cellcolor{blue!20}5.25 \\
Euclidean          &         12 &        4 &      5 &     5 &       4 &       10 &         9 &                   10 &    7 &   11 &             8 &     8 &   11 &       8 &          9 &    10 &         8.19 \\
Manhattan          &          9 &        2 &      5 &     8 &       1 &        8 &        11 &                   10 &    4 &   10 &             9 &     4 &    9 &       9 &          9 &     8 &         7.25 \\
Minkowski          &         11 &        2 &      5 &     7 &       6 &       11 &         6 &                   10 &    7 &   11 &             7 &     7 &   12 &       7 &          9 &    10 &         8 \\
Cosine             &          5 &        5 &     11 &     4 &       5 &        6 &         2 &                    4 &   12 &    7 &             4 &     9 &    4 &      11 &          2 &     7 &         \cellcolor{blue!10}6.12 \\
Canberra           &          3 &        6 &      4 &     9 &      10 &        3 &        10 &                    2 &   10 &    3 &             1 &    11 &    2 &       9 &         12 &     4 &         6.19 \\
Hellinger          &          2 &        7 &     12 &    10 &       9 &        4 &         4 &                    9 &    2 &    4 &            11 &    12 &    3 &       4 &          7 &     1 &         6.31 \\
Jensen-Shannon     &          6 &        9 &      8 &    12 &      12 &       12 &         5 &                    7 &    1 &    6 &            12 &     3 &    1 &       1 &          3 &    12 &         6.88 \\
PearsonChi2        &          6 &       12 &     10 &     6 &      11 &        2 &         1 &                    3 &    3 &    1 &             5 &     1 &    5 &       3 &          5 &     2 &          \cellcolor{blue!30}4.75 \\
VicisSymmetric1    &          1 &       11 &      9 &    11 &       8 &        5 &         8 &                    1 &   11 &    5 &            10 &    10 &    6 &       2 &          8 &     3 &         6.81 \\
Hassanat           &          4 &       10 &      1 &     1 &       7 &        1 &         3 &                    5 &    9 &    2 &             2 &     5 &   10 &       5 &          1 &     8 &         \cellcolor{red!40}4.62 \\\hline
\end{tabular}
\caption{Ranking of K-means Models by Recall with Silhouette}
\label{tab:rankings_recall_silhouette}
\end{table}

\textbf{Convergence.} The performance of Gini K-means is evaluated based on the number of iterations required for convergence. While the grid search used to obtain the optimal parameter ($\nu^*$) is computationally efficient, it is essential to analyze the convergence time, which is reported in terms of iteration counts. Compared to other methods, the optimized generalized Gini prametric (with $\nu^*$) exhibits an average of 6.5 iterations over the 16 datasets (and over the 5 folds), placing the method among the top five in terms of efficiency. 

\begin{table}[h!]
\centering
\scriptsize
\renewcommand{\arraystretch}{1}
\setlength{\tabcolsep}{2.5pt}
\begin{tabular}{|p{0.7in}|rrrrrrrrr}
\toprule
\textbf{Distances} &  Ir &  Wi &  Bc &  S &  A &  He &  Gl &  Ba &  G \\
\midrule
Gini pr. ($\nu=2$)               &   2.4 &   8.8 &    7.6 &   6.6 &   7.4 &  16.6 &   4.8 &   9.2 &   4.8 \\
Gini pr. ($\nu^*$) &   2.0 &  11.4 &    2.0 &   6.2 &   6.8 &   8.6 &   4.4 &   9.2 &   4.8 \\
Euclidean                   &   2.2 &   4.4 &    2.8 &   3.4 &   1.8 &   3.0 &   3.2 &  19.4 &  12.2 \\
Manhattan                   &   2.0 &   7.8 &    2.6 &   4.0 &   1.8 &   3.2 &   3.0 &  12.6 &  12.8 \\
Minkowski                   &   2.8 &   4.2 &    2.8 &   5.0 &   1.8 &   5.2 &   2.8 &  17.2 &  12.2 \\
Cosine                      &   1.6 &   1.0 &    1.0 &   2.6 &   1.0 &   1.0 &   2.8 &   3.2 &   1.0 \\
Hassanat                    &   2.6 &   6.2 &    4.0 &   6.0 &   1.0 & 165.8 &   5.2 &  22.0 &   9.2 \\
Canberra                    &   3.6 &   6.6 &    3.2 &   8.8 & 300.0 &  21.6 &   9.6 &  83.8 &  12.0 \\
Hellinger                   & 300.0 & 300.0 &  300.0 & 300.0 &   2.2 & 300.0 & 300.0 & 300.0 & 300.0 \\
Jensen\_Shannon              & 240.4 &   2.0 &    2.0 &   2.0 &   2.0 &   2.0 &   2.0 &   2.0 &   2.0 \\
Pearson Chi2                & 300.0 & 300.0 &    1.8 & 300.0 &   6.8 & 300.0 & 300.0 & 300.0 & 300.0 \\
Vicis Symmetric          & 300.0 & 300.0 &  240.8 & 300.0 & 300.0 & 300.0 & 300.0 & 300.0 & 300.0 \\
\bottomrule
\end{tabular}%
\caption{Steps to convergence of K-means: mean over the 5 folds (Datasets 1–9)}
\end{table}

\begin{table}[h!]
\centering
\scriptsize
\renewcommand{\arraystretch}{1}
\setlength{\tabcolsep}{2pt}
\begin{tabular}{|p{1in}|rrrrrrr|r}
\toprule
\textbf{Distances} &  Ha &  W &  V &  B &  In &  I &  QS & \cellcolor{gray!30} \textbf{Mean}  \\
\midrule
Gini prametric ($\nu=2$)                       &   3.8 &  18.4 &  12.4 &   8.4 &   7.4 &   4.0 &   5.2 &   8.0 \\
Gini prametric ($\nu^*$) &   5.0 &  14.2 &  10.8 &   5.4 &   6.0 &   2.4 &   4.6 &   \cellcolor{blue!10}6.5 \\
Euclidean                   &   5.6 &   4.8 &   3.4 &   2.0 &   2.2 &   2.4 &   3.6 &   \cellcolor{blue!30}4.8 \\
Manhattan                   &   5.0 &   7.4 &   4.2 &   4.0 &   2.2 &   3.2 &   7.2 &   \cellcolor{blue!20}5.2 \\
Minkowski                   &   5.2 &   5.0 &   4.4 &   2.0 &   2.2 &   2.2 &   2.4 &   \cellcolor{blue!30}4.8 \\
Cosine                      &   1.0 &   1.0 &   1.0 &   1.0 &   1.0 &   1.4 &   1.0 &   \cellcolor{red!40}1.4 \\
Hassanat                    &  19.0 &   1.0 &  18.2 &   6.0 &   4.0 &   5.4 &  11.4 &  17.9 \\
Canberra                    &  11.4 &   1.0 &  14.4 &   4.8 &   4.4 &   5.0 &  12.8 &  31.4 \\
Hellinger                   & 300.0 &   3.8 & 300.0 & 300.0 &   3.0 & 300.0 & 300.0 & 244.3 \\
Jensen\_Shannon              &   2.0 &  62.6 &   2.0 &   2.0 &   2.0 &   2.0 &   2.0 &  20.7 \\
Pearson Chi2                & 243.8 & 257.4 & 300.0 & 264.6 &  86.2 &   2.0 & 300.0 & 222.7 \\
Vicis Symmetric          & 300.0 &   2.2 & 300.0 & 300.0 & 300.0 & 300.0 & 300.0 & 277.7 \\
\bottomrule
\end{tabular}%
\caption{Steps to convergence of K-means: mean over the 5 folds (Datasets 10–16)}
\end{table}

\subsection{Illustration on Agglomerative Hierarchical Clustering}\label{sect:agglo}


For this additional experiment, noise has been directly added to the data using a standard normal distribution with a noise level of 10\% to evaluate which distance metric demonstrates greater robustness under noisy conditions. A 3-fold cross-validation has been made and the linkage criterion has been fixed to `average', i.e., the average distance between all points between clusters is used for all methods. The generalized Gini prametric is top-ranked with 7 top positions out of 16 datasets, and a mean rank of 1.94 (for precision).

\begin{table}[h!]
\centering
\tiny
\renewcommand{\arraystretch}{1} 
\setlength{\tabcolsep}{2pt} 
\begin{tabular}{|p{0.8in}|c|c|c|c|c|c|c|c|c|c|c|c|c|c|c|c|c|}
\hline
\textbf{Distances}   &  \textbf{I} & \textbf{B} & \textbf{G} & \textbf{He} & \textbf{Ba} & \textbf{Ha} & \textbf{W} & \textbf{In} &  \textbf{Ir} & \textbf{Wi} & \textbf{Bc} & \textbf{S} & \textbf{QS} & \textbf{V} & \textbf{A} & \textbf{Gl} &  \textbf{Rank} \\ \hline
Gini prametric $(\nu^*)$ &           7 &         1 &       2 &      2 &        2 &         1 &          1 &                     2 &     1 &     2 &              1 &      1 &           2 &      2 &     1 &        3 &           \cellcolor{red!40}1.94 \\
Gini prametric $(\nu=2)$            &           7 &        12 &       3 &      5 &       10 &         4 &         13 &                     2 &     3 &     6 &              3 &      6 &           5 &      3 &     7 &        9 &           6.12 \\
Euclidean        &           7 &         3 &      10 &      9 &       11 &         4 &          7 &                     2 &     4 &     9 &              6 &      3 &          10 &      5 &     5 &        2 &           6.06 \\
Euclidean Ward   &           2 &         6 &       1 &     13 &        1 &         3 &          2 &                     1 &     7 &     1 &              2 &      9 &           1 &      1 &     4 &        4 &           \cellcolor{blue!30}3.62 \\
Manhattan        &           6 &         2 &       4 &      9 &        8 &         4 &          7 &                     2 &     2 &     7 &              6 &      4 &          10 &     11 &     3 &        1 &           \cellcolor{blue!20}5.38 \\
Minkowski        &           7 &         4 &       5 &      9 &       13 &         4 &          7 &                     2 &     9 &     7 &              6 &     10 &          10 &      7 &     2 &        5 &           6.69 \\
Cosine           &           7 &         7 &       6 &      8 &        9 &         4 &          3 &                     2 &    10 &    11 &              5 &     10 &           9 &      8 &     8 &        8 &           7.19 \\
Canberra         &           5 &         8 &      12 &      3 &        3 &         4 &          7 &                     2 &    11 &    10 &             11 &      2 &          13 &     12 &     8 &        6 &           7.31 \\
Hellinger        &           7 &         5 &       7 &      7 &        7 &         4 &          5 &                     2 &     5 &     5 &              9 &     10 &           3 &      9 &     6 &       10 &           6.31 \\
Jensen\_Shannon    &           7 &        13 &      12 &      6 &        4 &         4 &          4 &                     2 &    11 &    13 &             11 &     10 &           7 &     13 &     8 &       13 &           8.62 \\
Pearson Chi2      &           3 &        10 &       8 &      9 &       12 &         2 &          6 &                     2 &     5 &     3 &             10 &      7 &           3 &      9 &     8 &       11 &           6.75 \\
Vicis Symmetric  &           4 &         9 &      11 &      4 &        6 &         4 &          7 &                     2 &    11 &    12 &             11 &      5 &           8 &      6 &     8 &       12 &           7.5 \\
Hassanat         &           1 &        11 &       9 &      1 &        5 &         4 &          7 &                     2 &     8 &     4 &              4 &      8 &           6 &      4 &     8 &        7 &           \cellcolor{blue!10}5.56 \\
\hline
\end{tabular}
\caption{Ranking of the Models by Precision: Noise 10\%}
\label{tab:rankings_precision_agglo_10}
\end{table}

\begin{table}[h!]
\centering
\tiny
\renewcommand{\arraystretch}{1} 
\setlength{\tabcolsep}{2pt} 
\begin{tabular}{|p{0.8in}|c|c|c|c|c|c|c|c|c|c|c|c|c|c|c|c|c|}
\hline
\textbf{Distances}   &  \textbf{I} & \textbf{B} & \textbf{G} & \textbf{He} & \textbf{Ba} & \textbf{Ha} & \textbf{W} & \textbf{In} &  \textbf{Ir} & \textbf{Wi} & \textbf{Bc} & \textbf{S} & \textbf{QS} & \textbf{V} & \textbf{A} & \textbf{Gl} &  \textbf{Rank} \\ \hline
Gini prametric $(\nu^*)$ &           7 &         1 &       2 &      3 &        8 &         3 &          1 &                     2 &     1 &     3 &              1 &      3 &           4 &      9 &     3 &        1 &           \cellcolor{blue!40}3.25 \\
Gini prametric $(\nu=2)$           &           7 &        12 &       4 &     13 &        1 &         3 &         13 &                     2 &     3 &     5 &              4 &      9 &           5 &      4 &    13 &       10 &           6.75 \\
Euclidean        &           7 &         4 &       6 &      8 &       11 &         3 &          7 &                     2 &     5 &    10 &              6 &      9 &          10 &     11 &     3 &        5 &           6.69 \\
Euclidean Ward   &           1 &         6 &       1 &      4 &        5 &         1 &          2 &                     1 &     2 &     1 &              2 &      1 &           1 &      6 &     6 &        2 &           \cellcolor{red!30}2.62 \\
Manhattan        &           6 &         2 &       3 &      8 &        3 &         3 &          7 &                     2 &     4 &     8 &              6 &     11 &          10 &     11 &     5 &        4 &           5.81 \\
Minkowski        &           7 &         3 &       5 &      8 &       13 &         3 &          7 &                     2 &    10 &     8 &              6 &      3 &          10 &      5 &     2 &        3 &           5.94 \\
Cosine           &           7 &         7 &       9 &      5 &        6 &         3 &          3 &                     2 &     7 &    11 &              5 &      3 &           9 &      7 &     7 &        9 &           6.25 \\
Canberra         &           3 &        10 &       9 &      2 &        4 &         3 &          7 &                     2 &    11 &     6 &             11 &      2 &          13 &     10 &     7 &       12 &           7 \\
Hellinger        &           7 &         5 &       9 &      5 &        2 &         3 &          5 &                     2 &     7 &     7 &              9 &      3 &           7 &      2 &     1 &        7 &           \cellcolor{blue!20}5.06 \\
Jensen\_Shannon    &           7 &        11 &       9 &      5 &        9 &         3 &          4 &                     2 &    11 &    12 &             11 &      3 &           2 &     13 &     7 &        8 &           7.31 \\
Pearson Chi2      &           5 &        13 &      13 &      8 &       12 &         2 &          6 &                     2 &     7 &     4 &             10 &     13 &           7 &      2 &     7 &        6 &           7.31 \\
Vicis Symmetric  &           2 &         9 &       8 &      8 &        7 &         3 &          7 &                     2 &    11 &    13 &             11 &     11 &           3 &      8 &     7 &       11 &           7.56 \\
Hassanat         &           3 &         8 &       7 &      1 &       10 &         3 &          7 &                     2 &     6 &     2 &              3 &      8 &           6 &      1 &     7 &       13 &           \cellcolor{blue!10}5.44 \\
\hline
\end{tabular}
\caption{Ranking of the Models by Recall: Noise 10\%}
\label{tab:rankings_recall_agglo10}
\end{table}

In this experiment, Ward's method has been added. It minimizes the within-cluster variance at each step of the clustering process. To be precise, it aims to minimize the increase in the sum of squared within-cluster distances (within-cluster inertia) when two clusters are merged. As can be seen in Table \ref{tab:rankings_recall_agglo10}, Ward's technique achieves the highest recall, ranking in the top position 7 times with an average rank of 2.62. The generalized Gini prametric ranks second (see also Tables 27-30 in Supplementary Materials \citep{mussard2025}).     

\section{Conclusion}\label{section:conclusion}

The generalized Gini prametric, while not being a true distance, satisfies two invariance properties, making it robust to outliers. In particular, it is linearly invariant (similar to the Euclidean distance) and rank invariant. Accordingly, monotonic increasing transformations preserve ranks, ensuring that the Gini prametric exhibits small variations. Consequently, when the data contain outliers or measurement errors, the generalized Gini prametric helps achieve robust results.

The experiments on the UCI datasets highlight the limitations of the generalized Gini prametric in unsupervised tasks. In particular, a strategy for selecting the hyper-parameter, such as using the silhouette score, must be chosen without prior knowledge of the quality of the results. This remains an open problem.


\newpage

\bibliography{mybibfile}

\pagebreak
\onecolumn

\begin{center}
    \Huge Supplementary Materials
\end{center}

\setcounter{table}{0}

\section*{Proof of Proposition 4}

\noindent Let $\textbf{c}_{k}^{(t-1)}$ be the centroid of group $k$ at iteration $(t-1)$. If point $i$ goes from cluster $ C_k^{(t-1)}$ at iteration $(t-1)$ to $ C_{k}^{(t)}$ at iteration $(t)$ then: 
\[
d_{G,\nu}\big(\x_i, \textbf{c}_{k}^{(t)}\big)^2 \leq d_{G,\nu}\big(\x_i, \textbf{c}_{k}^{(t-1)}\big)^2 
\]
The last inequality comes from the algorithm itself. Indeed, if $i \in C_k^{(t)}$, this is because at iteration $t-1$ we have $k = \arg\underset{ 1 \leq j \leq K}{\mathrm{min}} d_{G,\nu}(\x_i, \textbf{c}_{j}^{(t-1)})$. Now, we must prove that taking the sum over $i$ on the last expression implies:
\begin{align}
        \sum_{i \in C_k^{(t)}} d_{G,\nu}\big(\x_i, \textbf{c}_k^{(t)}\big)^2 &= \sum_{i \in C_k^{(t)}}  \Big[-\sum_{j=1}^{d} \big(x_{ij} - c_{jk}^{(t)}\big)\big( \overline{R}_X(x_{ij})^{\nu-1} - \overline{R}_X(c_{jk}^{(t)})^{\nu-1}\big)\Big]^2 \notag \\
        &\leq  \sum_{i \in C_k^{(t)}} \Big[ -\sum_{j=1}^{d} \big(x_{ij} - c_{jk}^{(t-1)}\big)\big(\overline{R}_X(x_{ij})^{\nu-1}-\overline{R}_X(c_{jk}^{(t-1)})^{\nu-1}\big)\Big]^2 \label{proof-convergence}  \\
        &= \sum_{i \in C_k^{(t)}} d_{G,\nu}\big(\x_i, \textbf{c}_k^{(t-1)}\big)^2\notag
\end{align}
Let us prove, for all points $i$ in a given cluster $C$, that the mean $\bar{\mathbf z} \in \mathbb R^d$ is the unique point in $\mathbb R^d$ implying a minimal distance, \textit{i.e.},
\[
\overline{\mathbf{z}} = \underset{\mathbf{z}}{\mathrm{\arg\min}} \sum_{i \in C} \Big[-\sum_{j=1}^{d} (z_{ij} - z_j)(\overline{R}_Z(z_{ij})^{\nu-1}-\overline{R}_Z(z_j)^{\nu-1}) \Big]^2 \equiv  \underset{\textbf{z}}{\mathrm{\arg\min}} f({\mathbf z}) 
\]
with $\mathbf{z} = (z_1,\cdots,z_{j},\cdots,z_d)$. 
Taking the derivative of $f$ with respect to $\mathbf{z}$, and imposing that for small variations of $z_{j}$ ranks remain constant, yields 
\[
f'(\textbf{z}) = 0 \ \Longrightarrow \ 2\sum_{i \in C} \sum_{j=1}^{d} \big(z_{ij} - z_j\big)\big(\overline{R}_Z(z_{ij})^{\nu-1}-\overline{R}_Z(z_j)^{\nu-1}\big) = 0
\]
Using the scalar product this is equivalent to:
\[
2\sum_{i \in C}  \big(\textbf{z}_{i} - \textbf{z}\big) \cdot \big(\overline{R}_Z(\textbf{z}_{i})^{\nu-1}-\overline{R}_Z(\textbf{z})^{\nu-1}\big) = 0
\]
Thus,
\[
\sum_{i \in C}  \textbf{z}_{i} \cdot  \big(\overline{R}_Z(\textbf{z}_{i})^{\nu-1}-\overline{R}_Z(\textbf{z})^{\nu-1}\big) = \sum_{i\in C} \textbf{z} \cdot  \big(\overline{R}_Z(\textbf{z}_{i})^{\nu-1}-\overline{R}_Z(\textbf{z})^{\nu-1}\big) 
\]
This implies,
\[
\textbf{z} = \frac{1}{n} \sum_{i \in C} \textbf{z}_{i} = \bar{\textbf{z}}
\]
Consequently, the unique centroid $\mathbf c$ that minimizes the distance with all points $i$ in cluster $C$ is the arithmetic mean of points $i$. Therefore, 
Eq.\eqref{proof-convergence} is always true if $\overline{R}_Z(\textbf{z}_{i})^{\nu-1}-\overline{R}_Z(\bar{\textbf{z}})^{\nu-1} \neq \mathbf{0}$.

\break

\section*{Experiments}

In what follows, the results of the Minkowski distance are given for $p=3$.

\subsection*{KNN experiments}

\subsubsection*{Without noise}

\begin{table*}[!htbp]
\centering
\caption{Precision (\#P), Recall (\#R), Best $k$ that maximizes the F1-score and $\nu$ parameter on KNN algorithm for selected UCI datasets.}
\vspace{3mm}
\makebox[\textwidth][c]{
\resizebox{\textwidth}{!}{
\begin{tabular}{|p{2.8cm}||p{1.5cm}|p{1cm}|p{1cm}||p{1.5cm}|p{1cm}|p{1cm}||p{1.5cm}|p{1cm}|p{1cm}||p{1.5cm}|p{1cm}|p{1cm}||p{1.5cm}|p{1cm}|p{1cm}||p{1.5cm}|p{1cm}|p{1cm}||p{1cm}|p{1.5cm}|p{1cm}||p{1cm}|p{1cm}|p{1.5cm}|}
\hline
\multicolumn{1}{|c||}{} & \multicolumn{3}{c||}{Ionosphere} & \multicolumn{3}{c||}{Banknote} & \multicolumn{3}{c||}{German} & \multicolumn{3}{c||}{Heart} & \multicolumn{3}{c||}{Balance} & \multicolumn{3}{c||}{Haberman} & \multicolumn{3}{c||}{Wholesale} & \multicolumn{3}{c||}{Indian Liver Patient} \\
\hline
Distance & $k$ \hspace{0.5cm} $\nu$ & \#P & \#R & $k$ \hspace{0.5cm} $\nu$ & \#P & \#R & $k$ \hspace{0.5cm} $\nu$ & \#P & \#R & $k$ \hspace{0.5cm} $\nu$ & \#P & \#R & $k$ \hspace{0.5cm} $\nu$ & \#P & \#R & $k$ \hspace{0.5cm} $\nu$ & \#P & \#R & $k$ \ \ $\nu$ & \#P & \#R & $k$ \ \ $\nu$ & \#P & \#R \\
\hline
Euclidean & 2   & 0.89 & 0.89 & 2   & \cellcolor{blue!40}1.00 & \cellcolor{blue!40}1.00 & 2   & 0.68 & 0.65 & 1   & 0.56 & 0.57 & 10   & 0.84 & 0.88 & 8   & 0.75 & 0.75 & 9   & 0.91 & 0.90 & 2   & 0.60 & \cellcolor{blue!40}0.37 \\
Manhattan & 1   & 0.91 & 0.90 & 1   & \cellcolor{blue!40}1.00 & \cellcolor{blue!40}1.00 & 2   & 0.70 & 0.66 & 4   & 0.56 & 0.58 & 8   & 0.86 & 0.87 & 10   & 0.74 & \cellcolor{blue!40}0.76 & 9   & 0.91 & 0.91 & 2   & 0.60 & \cellcolor{blue!40}0.37 \\
Minkowski & 2   & 0.88 & 0.87 & 2   & \cellcolor{blue!40}1.00 & \cellcolor{blue!40}1.00 & 2   & 0.68 & 0.65 & 3   & 0.56 & 0.50 & 10   & 0.85 & 0.88 & 8   & \cellcolor{blue!40}0.76 & \cellcolor{blue!40}0.76 & 9   & 0.90 & 0.89 & 2   & 0.59 & 0.36 \\
Cosine & 2   & 0.89 & 0.89 & 1   & 0.99 & 0.99 & 8   & 0.69 & 0.69 & 6   & 0.52 & 0.55 & 3   & 0.93 & 0.90 & 2   & 0.74 & \cellcolor{blue!40}0.76 & 10   & 0.88 & 0.88 & 1   & 0.55 & 0.36 \\
\hline
\end{tabular}}}
\label{res_knn_selected}
\end{table*}

\begin{table*}[!htbp]
\centering

\caption{Precision (\#P), Recall (\#R), Best $k$ that maximizes the F1-score and $\nu$ parameter on KNN algorithm for various distances on UCI datasets.}
\vspace{3mm}
\makebox[\textwidth][c]{
\resizebox{\textwidth}{!}{
\begin{tabular}{|p{2.8cm}||p{1.5cm}|p{1cm}|p{1cm}||p{1.5cm}|p{1cm}|p{1cm}||p{1.5cm}|p{1cm}|p{1cm}||p{1.5cm}|p{1cm}|p{1cm}||p{1.5cm}|p{1cm}|p{1cm}||p{1.5cm}|p{1cm}|p{1cm}||p{1.5cm}|p{1cm}|p{1cm}||p{1.5cm}|p{1cm}|p{1cm}||}
\hline
\multicolumn{1}{|c||}{} & \multicolumn{3}{c||}{Iris} & \multicolumn{3}{c||}{Wine} & \multicolumn{3}{c||}{Breast-Cancer} & \multicolumn{3}{c||}{Sonar} & \multicolumn{3}{c||}{QSAR} & \multicolumn{3}{c||}{Vehicle} & \multicolumn{3}{c||}{Australian} & \multicolumn{3}{c|}{Glass} \\
\hline
Distance & $k$ \hspace{0.5cm} $\nu$ & \#P & \#R & $k$ \hspace{0.5cm} $\nu$ & \#P & \#R & $k$ \hspace{0.5cm} $\nu$ & \#P & \#R & $k$ \hspace{0.5cm} $\nu$ & \#P & \#R & $k$ \hspace{0.5cm} $\nu$ & \#P & \#R & $k$ \hspace{0.5cm} $\nu$ & \#P & \#R & $k$ \hspace{0.5cm} $\nu$ & \#P & \#R & $k$ \hspace{0.5cm} $\nu$ & \#P & \#R \\
\hline
Euclidean & 5 \hspace{0.5cm} - & 0.97 & 0.97 & 1 \hspace{0.5cm} - & 0.71 & 0.71 & 10 \hspace{0.5cm} - & 0.93 & 0.93 & 1 \hspace{0.5cm} - & 0.84 & 0.83 & 6 \hspace{0.5cm} - & 0.64 & 0.65 & 1 \hspace{0.5cm} - & 0.87 & 0.87 & 6 \hspace{0.5cm} - & 0.70 & 0.69 & 3 \hspace{0.5cm} - & 0.74 & 0.72 \\
Manhattan & 1 \hspace{0.5cm} - & 0.96 & 0.96 & 1 \hspace{0.5cm} - & 0.77 & 0.77 & 6 \hspace{0.5cm} - & 0.94 & 0.94 & 1 \hspace{0.5cm} - & 0.86 & 0.86 & 9 \hspace{0.5cm} - & 0.64 & 0.65 & 1 \hspace{0.5cm} - & 0.90 & 0.90 & 8 \hspace{0.5cm} - & 0.75 & 0.73 & 1 \hspace{0.5cm} - & 0.75 & 0.71 \\
Minkowski & 5 \hspace{0.5cm} - & 0.97 & 0.97 & 1 \hspace{0.5cm} - & 0.68 & 0.69 & 10 \hspace{0.5cm} - & 0.93 & 0.93 & 3 \hspace{0.5cm} - & 0.84 & 0.83 & 8 \hspace{0.5cm} - & 0.63 & 0.64 & 1 \hspace{0.5cm} - & 0.85 & 0.85 & 10 \hspace{0.5cm} - & 0.70 & 0.67 & 2 \hspace{0.5cm} - & 0.68 & 0.65 \\
\hline
\end{tabular}}}
\label{res_knn}
\end{table*}

\subsubsection*{With noise at level 5\%}

\begin{table*}[!htbp]
\centering
\caption{Precision (\#P), Recall (\#R), Best $k$ that maximizes the F1-score and $\nu$ parameter on KNN algorithm for various distances on UCI datasets with 5\% noise.}
\resizebox{\textwidth}{!}{
\begin{tabular}{|p{2.8cm}||p{1.5cm}|p{1cm}|p{1cm}||p{1.5cm}|p{1cm}|p{1cm}||p{1.5cm}|p{1cm}|p{1cm}||p{1.5cm}|p{1cm}|p{1cm}||p{1.5cm}|p{1cm}|p{1cm}||p{1.5cm}|p{1cm}|p{1cm}||p{1.5cm}|p{1cm}|p{1cm}||p{1.5cm}|p{1cm}|p{1cm}|}
\hline
\multicolumn{1}{|c||}{} & \multicolumn{3}{c||}{Ionosphere} & \multicolumn{3}{c||}{Banknote} & \multicolumn{3}{c||}{German} & \multicolumn{3}{c||}{Heart} & \multicolumn{3}{c||}{Balance} & \multicolumn{3}{c||}{Haberman} & \multicolumn{3}{c||}{Wholesale} & \multicolumn{3}{c|}{Indian Liver Patient} \\
\hline
Distance & $k$ \hspace{0.5cm} $\nu$ & \#P & \#R & $k$ \hspace{0.5cm} $\nu$ & \#P & \#R & $k$ \hspace{0.5cm} $\nu$ & \#P & \#R & $k$ \hspace{0.5cm} $\nu$ & \#P & \#R & $k$ \hspace{0.5cm} $\nu$ & \#P & \#R & $k$ \hspace{0.5cm} $\nu$ & \#P & \#R & $k$ \hspace{0.5cm} $\nu$ & \#P & \#R & $k$ \hspace{0.5cm} $\nu$ & \#P & \#R \\
\hline
Euclidean & 2   & 0.885 & 0.883 & 2   & \cellcolor{blue!40}1.000 &  \cellcolor{blue!40}1.000 & 2   & 0.683 & 0.647 & 1   & 0.563 & 0.568 & 9   & 0.840 & 0.883 & 8   & 0.747 & 0.765 & 1   & 0.905 & 0.902 & 1   & 0.682 & 0.668 \\
Manhattan & 2   & 0.901 & 0.900 & 1   &  \cellcolor{blue!40}1.000 & \cellcolor{blue!40}1.000 & 2   & 0.701 & 0.663 & 4   & 0.568 & 0.594 & 8   & 0.830 & 0.875 & 10   & 0.745 & 0.761 & 9   & 0.911 & 0.909 & 1   & 0.685 & 0.668 \\
Minkowski & 2   & 0.889 & 0.886 & 2   & \cellcolor{blue!40}1.000 & \cellcolor{blue!40}1.000 & 4   & 0.687 & 0.664 & 1   & 0.545 & 0.554 & 8   & 0.838 & 0.869 & 8   & 0.750 & 0.761 & 9   & 0.895 & 0.893 & 1   & 0.685 & 0.672 \\
Cosine & 2   & \cellcolor{blue!40} 0.907 & \cellcolor{blue!40}0.903 & 1   & 0.998 & 0.998 & 8   & 0.696 & 0.690 & 6   & 0.541 & 0.564 & 3   & \cellcolor{blue!40}0.916 & \cellcolor{blue!40}0.907 & 9   & 0.738 & 0.754 & 10   & 0.881 & 0.877 & 1   & 0.670 & 0.668 \\
\hline
\end{tabular}}
\label{res_knn_noise_5}
\end{table*}

\begin{table*}[!htbp]
\centering
\caption{Precision (\#P), Recall (\#R), Best $k$ that maximizes the F1-score and $\nu$ parameter on KNN algorithm for various distances on UCI datasets with 5\% noise.}
\resizebox{\textwidth}{!}{
\begin{tabular}{|p{2.8cm}||p{1.5cm}|p{1cm}|p{1cm}||p{1.5cm}|p{1cm}|p{1cm}||p{1.5cm}|p{1cm}|p{1cm}||p{1.5cm}|p{1cm}|p{1cm}||p{1.5cm}|p{1cm}|p{1cm}||p{1.5cm}|p{1cm}|p{1cm}||p{1.5cm}|p{1cm}|p{1cm}||p{1.5cm}|p{1cm}|p{1cm}||}
\hline
\multicolumn{1}{|c||}{} & \multicolumn{3}{c||}{Iris} & \multicolumn{3}{c||}{Wine} & \multicolumn{3}{c||}{Breast-Cancer} & \multicolumn{3}{c||}{Sonar} & \multicolumn{3}{c||}{QSAR} & \multicolumn{3}{c||}{Vehicle} & \multicolumn{3}{c||}{Australian} & \multicolumn{3}{c|}{Glass} \\
\hline
Distance & $k$ \hspace{0.5cm} $\nu$ & \#P & \#R & $k$ \hspace{0.5cm} $\nu$ & \#P & \#R & $k$ \hspace{0.5cm} $\nu$ & \#P & \#R & $k$ \hspace{0.5cm} $\nu$ & \#P & \#R & $k$ \hspace{0.5cm} $\nu$ & \#P & \#R & $k$ \hspace{0.5cm} $\nu$ & \#P & \#R & $k$ \hspace{0.5cm} $\nu$ & \#P & \#R & $k$ \hspace{0.5cm} $\nu$ & \#P & \#R \\
\hline
Euclidean & 9   & 0.982 & 0.980 & 1   & 0.706 & 0.707 & 10   & 0.935 & 0.935 & 1   & 0.834 & 0.832 & 6   & 0.644 & 0.653 & 1   & 0.865 & 0.863 & 9   & 0.711 & 0.703 & 1   & 0.702 & 0.678 \\
Manhattan & 10   & 0.976 & 0.973 & 1   & 0.776 & 0.775 & 10   & 0.941 & 0.940 & 3   & 0.845 & 0.845 & 9   & 0.633 & 0.643 & 1   & 0.900 & 0.898 & 8   & 0.745 & 0.727 & 1   & 0.741 & 0.701 \\
Minkowski & 5   & 0.980 & 0.980 & 1   & 0.682 & 0.685 & 10   & 0.933 & 0.933 & 3   & 0.805 & 0.803 & 6 \hspace{0.5cm} 3 & 0.646 & 0.653 & 1   & 0.849 & 0.846 & 10   & 0.703 & 0.696 & 1   & 0.707 & 0.696 \\
Cosine & 3   & 0.982 & 0.980 & 1   & 0.828 & 0.826 & 7   & 0.927 & 0.924 & 1   & 0.851 & 0.846 & 7   & 0.618 & 0.623 & 1   & 0.886 & 0.884 & 9   & 0.716 & 0.716 & 1   & 0.720 & 0.710 \\
Hassanat & 4   & 0.968 & 0.967 & 7   & 0.960 & 0.955 & 6   & \cellcolor{blue!40}0.955 & \cellcolor{blue!40}0.954 & 1   & 0.823 & 0.817 & 10   & 0.642 & 0.651 & 7   & 0.886 & 0.884 & 8   & 0.851 & 0.844 & 4   & 0.681 & 0.645 \\
\hline
\end{tabular}}
\label{res_knn_noise_5_2}
\end{table*}

\begin{table}[!htbp]
\centering

\caption{Ranking of KNN Models by Precision}
\vspace{3mm}
\tiny
\renewcommand{\arraystretch}{1} 
\setlength{\tabcolsep}{1.5pt} 
\begin{tabular}{|p{0.9in}|c|c|c|c|c|c|c|c|c|c|c|c|c|c|c|c|c|}
\hline
\textbf{Distances}   &  \textbf{I} & \textbf{B} & \textbf{G} & \textbf{He} & \textbf{Ba} & \textbf{Ha} & \textbf{W} & \textbf{In} &  \textbf{Ir} & \textbf{Wi} & \textbf{Bc} & \textbf{S} & \textbf{QS} & \textbf{V} & \textbf{A} & \textbf{Gl} &  \textbf{Rank} \\ \hline
Euclidean        &           9 &         1 &      10 &     10 &        7 &         3 &          9 &                     8 &     2 &    11 &              8 &      7 &     6 &        7 &          10 &      6 &           7.12 \\
Manhattan        &           4 &         1 &       6 &      9 &        9 &         4 &          7 &                     6 &     8 &    10 &              6 &      6 &     9 &        3 &           8 &      2 &           6.12 \\
Minkowski        &           8 &         1 &       9 &     11 &        8 &         2 &         10 &                     6 &     7 &    12 &              9 &      9 &     5 &       10 &          12 &      5 &           7.75 \\
Cosine           &           1 &         7 &       8 &     12 &        1 &         9 &         12 &                    12 &     2 &     9 &             10 &      4 &    10 &        4 &           9 &      4 &           7.12 \\
Hassanat         &           2 &         5 &       1 &      3 &       11 &         9 &          8 &                     3 &    11 &     2 &              1 &      8 &     7 &        4 &           2 &      7 &           \cellcolor{blue!20}5.25 \\
Canberra         &           5 &         8 &       7 &      2 &       10 &         6 &          6 &                    11 &    11 &     3 &              7 &     10 &     8 &        6 &           1 &     12 &           7.06 \\
Hellinger        &           7 &        10 &       4 &      6 &        5 &         7 &          2 &                     2 &     5 &     5 &              3 &      3 &     4 &        9 &           5 &      9 &           \cellcolor{blue!10}5.38 \\
\hline
\end{tabular}
\label{tab:rankings_precision_knn_5}
\end{table}

\begin{table}[!htbp]
\centering
\caption{Ranking of KNN Models by Recall}
\vspace{3mm}
\tiny
\renewcommand{\arraystretch}{1} 
\setlength{\tabcolsep}{1.5pt} 
\begin{tabular}{|p{0.9in}|c|c|c|c|c|c|c|c|c|c|c|c|c|c|c|c|c|}
\hline
\textbf{Distances}   &  \textbf{I} & \textbf{B} & \textbf{G} & \textbf{He} & \textbf{Ba} & \textbf{Ha} & \textbf{W} & \textbf{In} &  \textbf{Ir} & \textbf{Wi} & \textbf{Bc} & \textbf{S} & \textbf{QS} & \textbf{V} & \textbf{A} & \textbf{Gl} &  \textbf{Rank} \\ \hline
Euclidean        &           9 &         1 &      10 &      9 &        2 &         2 &          8 &                    10 &     2 &    11 &              8 &      7 &     5 &        8 &          10 &      6 &           6.75 \\
Manhattan        &           3 &         1 &       9 &      7 &        4 &         6 &          7 &                    10 &     8 &    10 &              6 &      5 &     8 &        3 &           8 &      3 &           6.12 \\
Minkowski        &           8 &         1 &       8 &     11 &        8 &         6 &         10 &                     9 &     2 &    12 &              9 &      9 &     5 &       10 &          11 &      4 &           7.69 \\
Cosine           &           1 &         7 &       7 &     10 &        1 &         8 &         11 &                    10 &     2 &     9 &             10 &      3 &    10 &        4 &           9 &      1 &           6.44 \\
Hassanat         &           1 &         5 &       2 &      2 &        6 &         4 &          8 &                     4 &    11 &     2 &              1 &      8 &     7 &        4 &           2 &      9 &           \cellcolor{blue!30}4.75 \\
Canberra         &           7 &         8 &       6 &      1 &       12 &         4 &          6 &                     3 &    11 &     2 &              7 &     11 &     9 &        6 &           1 &     11 &           6.56 \\
Hellinger        &           5 &        10 &       4 &      3 &        5 &        10 &          2 &                     6 &     2 &     7 &              3 &      3 &     4 &        9 &           5 &      7 &           \cellcolor{blue!20}5.31 \\
\hline
\end{tabular}
\label{tab:rankings_recall_knn_5}
\end{table}

\newpage
\subsubsection*{With noise at level 10\%}

\begin{table*}[!htbp]
\centering
\caption{Precision (\#P), Recall (\#R), Best $k$ that maximizes the F1-score and $\nu$ parameter on KNN algorithm for various distances on UCI datasets with 10\% noise.}
\resizebox{\textwidth}{!}{
\begin{tabular}{|p{2.8cm}||p{1.5cm}|p{1cm}|p{1cm}||p{1.5cm}|p{1cm}|p{1cm}||p{1.5cm}|p{1cm}|p{1cm}||p{1.5cm}|p{1cm}|p{1cm}||p{1.5cm}|p{1cm}|p{1cm}||p{1.5cm}|p{1cm}|p{1cm}||p{1.5cm}|p{1cm}|p{1cm}||p{1.5cm}|p{1cm}|p{1cm}|}
\hline
\multicolumn{1}{|c||}{} & \multicolumn{3}{c||}{Ionosphere} & \multicolumn{3}{c||}{Banknote} & \multicolumn{3}{c||}{German} & \multicolumn{3}{c||}{Heart} & \multicolumn{3}{c||}{Balance} & \multicolumn{3}{c||}{Haberman} & \multicolumn{3}{c||}{Wholesale} & \multicolumn{3}{c|}{Indian Liver Patient} \\
\hline
Distance & $k$ \hspace{0.5cm} $\nu$ & \#P & \#R & $k$ \hspace{0.5cm} $\nu$ & \#P & \#R & $k$ \hspace{0.5cm} $\nu$ & \#P & \#R & $k$ \hspace{0.5cm} $\nu$ & \#P & \#R & $k$ \hspace{0.5cm} $\nu$ & \#P & \#R & $k$ \hspace{0.5cm} $\nu$ & \#P & \#R & $k$ \hspace{0.5cm} $\nu$ & \#P & \#R & $k$ \hspace{0.5cm} $\nu$ & \#P & \#R \\
\hline
Euclidean & 2   & 0.885 & 0.880 & 2   & \cellcolor{blue!40}1.000 & \cellcolor{blue!40}1.000 & 2   & 0.687 & 0.652 & 1   & 0.565 & 0.571 & 8   & 0.850 & 0.875 & 8   & 0.747 & 0.765 & 9   & 0.902 & 0.902 & 1   & 0.680 & 0.665 \\
Manhattan & 2   & 0.900 & 0.900 & 1   & 0.999 & 0.999 & 2   & 0.701 & 0.665 & 4   & 0.565 & 0.594 & 9   & 0.815 & 0.875 & 10   & 0.743 & 0.761 & 8   & 0.911 & 0.909 & 1   & 0.678 & 0.663 \\
Minkowski & 2   & 0.888 & 0.886 & 2   & \cellcolor{blue!40}1.000 & \cellcolor{blue!40}1.000 & 4   & 0.686 & 0.663 & 1   & 0.549 & 0.558 & 8   & 0.843 & 0.869 & 8   & \cellcolor{blue!40}0.764 & \cellcolor{blue!40}0.775 & 9   & 0.895 & 0.893 & 1   & 0.685 & 0.671 \\
Cosine & 2   & \cellcolor{blue!40}0.906 & 0.897 & 1   & 0.999 & 0.999 & 8   & 0.696 & 0.692 & 6   & 0.540 & 0.564 & 1   & \cellcolor{blue!40}0.916 & \cellcolor{blue!40}0.907 & 9   & 0.736 & 0.758 & 10   & 0.882 & 0.877 & 1   & 0.671 & 0.668 \\
Hassanat & 3   & 0.902 & \cellcolor{blue!40}0.903 & 7   & 0.999 & 0.999 & 9   & 0.712 & 0.713 & 5   & 0.657 & 0.701 & 8   & 0.820 & 0.872 & 4   & 0.738 & 0.764 & 9   & 0.910 & 0.909 & 1   & 0.680 & 0.668 \\
\hline
\end{tabular}}
\label{res_knn_noise_10}
\end{table*}

\begin{table*}[!htbp]
\centering
\caption{Precision (\#P), Recall (\#R), Best $k$ that maximizes the F1-score and $\nu$ parameter on KNN algorithm for various distances on UCI datasets with 10\% noise.}
\resizebox{\textwidth}{!}{
\begin{tabular}{|p{2.8cm}||p{1.5cm}|p{1cm}|p{1cm}||p{1.5cm}|p{1cm}|p{1cm}||p{1.5cm}|p{1cm}|p{1cm}||p{1.5cm}|p{1cm}|p{1cm}||p{1.5cm}|p{1cm}|p{1cm}||p{1.5cm}|p{1cm}|p{1cm}||p{1.5cm}|p{1cm}|p{1cm}||p{1.5cm}|p{1cm}|p{1cm}||}
\hline
\multicolumn{1}{|c||}{} & \multicolumn{3}{c||}{Iris} & \multicolumn{3}{c||}{Wine} & \multicolumn{3}{c||}{Breast-Cancer} & \multicolumn{3}{c||}{Sonar} & \multicolumn{3}{c||}{QSAR} & \multicolumn{3}{c||}{Vehicle} & \multicolumn{3}{c||}{Australian} & \multicolumn{3}{c|}{Glass} \\
\hline
Distance & $k$ \hspace{0.5cm} $\nu$ & \#P & \#R & $k$ \hspace{0.5cm} $\nu$ & \#P & \#R & $k$ \hspace{0.5cm} $\nu$ & \#P & \#R & $k$ \hspace{0.5cm} $\nu$ & \#P & \#R & $k$ \hspace{0.5cm} $\nu$ & \#P & \#R & $k$ \hspace{0.5cm} $\nu$ & \#P & \#R & $k$ \hspace{0.5cm} $\nu$ & \#P & \#R & $k$ \hspace{0.5cm} $\nu$ & \#P & \#R \\
\hline
Euclidean & 7   & 0.967 & 0.967 & 1   & 0.706 & 0.707 & 10   & 0.933 & 0.933 & 1   & 0.848 & 0.846 & 6   & 0.637 & 0.648 & 1   & 0.868 & 0.865 & 10   & 0.707 & 0.699 & 1   & 0.703 & 0.678 \\
Manhattan & 3   & 0.967 & 0.967 & 1   & 0.770 & 0.769 & 10   & 0.943 & 0.942 & 1   & 0.851 & 0.846 & 9   & 0.647 & 0.655 & 1   & 0.902 & 0.899 & 8   & 0.744 & 0.726 & 1   & 0.678 & 0.654 \\
Minkowski & 7   & 0.967 & 0.967 & 1   & 0.677 & 0.679 & 10   & 0.933 & 0.933 & 1   & 0.838 & 0.832 & 5 \hspace{0.5cm} 3 & 0.621 & 0.633 & 1   & 0.851 & 0.847 & 10   & 0.707 & 0.699 & 1   & 0.725 & \cellcolor{blue!40}0.696 \\
Cosine & 7   & 0.965 & 0.960 & 1   & 0.835 & 0.831 & 7   & 0.924 & 0.924 & 1   & 0.838 & 0.832 & 4   & 0.622 & 0.634 & 1   & 0.884 & 0.884 & 9   & 0.716 & 0.715 & 1   & 0.720 & 0.710 \\
Hassanat & 3   & 0.960 & 0.960 & 7   & \cellcolor{blue!40}0.960 & \cellcolor{blue!40}0.955 & 6   & \cellcolor{blue!40}0.955 & \cellcolor{blue!40}0.954 & 1   & 0.840 & 0.837 & 7   & 0.645 & 0.653 & 7   & 0.884 & 0.882 & 8   & 0.844 & 0.841 & 4   & 0.685 & 0.645 \\
\hline
\end{tabular}}
\label{res_knn_noise_10_2}
\end{table*}

\begin{table}[!htbp]
\centering
\caption{Ranking of KNN Models by Precision}
\vspace{3mm}
\tiny
\renewcommand{\arraystretch}{1} 
\setlength{\tabcolsep}{1.5pt} 
\begin{tabular}{|p{0.9in}|c|c|c|c|c|c|c|c|c|c|c|c|c|c|c|c|c|}
\hline
\textbf{Distances}   &  \textbf{I} & \textbf{B} & \textbf{G} & \textbf{He} & \textbf{Ba} & \textbf{Ha} & \textbf{W} & \textbf{In} &  \textbf{Ir} & \textbf{Wi} & \textbf{Bc} & \textbf{S} & \textbf{QS} & \textbf{V} & \textbf{A} & \textbf{Gl} &  \textbf{Rank} \\ \hline
Euclidean        &           9 &         1 &       9 &      9 &        5 &         2 &          9 &                     6 &     2 &    11 &              8 &      3 &     7 &        7 &          11 &      5 &           6.5 \\
Manhattan        &           4 &         4 &       6 &      9 &       11 &         4 &          7 &                     9 &     2 &    10 &              6 &      2 &     4 &        3 &           8 &      7 &           6 \\
Minkowski        &           7 &         1 &      10 &     11 &        8 &         1 &         10 &                     5 &     2 &    12 &              8 &      5 &    10 &       10 &          11 &      3 &           7.12 \\
Cosine           &           1 &         4 &       7 &     12 &        1 &         9 &         12 &                    12 &     6 &     9 &             10 &      5 &     9 &        4 &           9 &      4 &           7.12 \\
Hassanat         &           2 &         4 &       2 &      3 &       10 &         7 &          8 &                     6 &     7 &     1 &              1 &      4 &     5 &        4 &           2 &      6 &           \cellcolor{blue!30}4.5 \\
Canberra         &           6 &         8 &       7 &      1 &        9 &         5 &          5 &                    11 &     7 &     2 &              7 &     10 &     8 &        6 &           1 &     12 &           6.56 \\
Hellinger        &           2 &        11 &       2 &      5 &        5 &         8 &          2 &                     1 &     7 &     5 &              3 &      5 &     6 &        9 &           5 &      9 &           \cellcolor{blue!10}5.31 \\
\hline
\end{tabular}
\label{tab:rankings_precision_knn_10}
\end{table}

\begin{table}[!htbp]
\centering
\caption{Ranking of KNN Models by Recall}
\vspace{3mm}
\tiny
\renewcommand{\arraystretch}{1} 
\setlength{\tabcolsep}{1.5pt} 
\begin{tabular}{|p{0.9in}|c|c|c|c|c|c|c|c|c|c|c|c|c|c|c|c|c|}
\hline
\textbf{Distances}   &  \textbf{I} & \textbf{B} & \textbf{G} & \textbf{He} & \textbf{Ba} & \textbf{Ha} & \textbf{W} & \textbf{In} &  \textbf{Ir} & \textbf{Wi} & \textbf{Bc} & \textbf{S} & \textbf{QS} & \textbf{V} & \textbf{A} & \textbf{Gl} &  \textbf{Rank} \\ \hline
Euclidean        &           9 &         1 &      10 &      9 &        2 &         2 &          9 &                    11 &     2 &    11 &              8 &      2 &     6 &        7 &          10 &      4 &           6.44 \\
Manhattan        &           3 &         4 &       8 &      7 &        2 &         6 &          7 &                    12 &     2 &    10 &              6 &      2 &     4 &        3 &           8 &      7 &           5.69 \\
Minkowski        &           7 &         1 &       9 &     11 &        8 &         1 &         10 &                     8 &     2 &    12 &              8 &      7 &     9 &       10 &          10 &      2 &           7.19 \\
Cosine           &           4 &         4 &       7 &     10 &        1 &         7 &         11 &                     9 &     6 &     9 &             10 &      7 &     8 &        4 &           9 &      1 &           6.69 \\
Hassanat         &           1 &         4 &       1 &      2 &        6 &         4 &          7 &                     9 &     6 &     1 &              1 &      4 &     5 &        5 &           2 &      9 &           \cellcolor{blue!30}4.19 \\
Canberra         &           6 &         8 &       6 &      1 &       12 &         4 &          6 &                     2 &     6 &     1 &              7 &     11 &    10 &        6 &           1 &     11 &           6.12 \\
Hellinger        &           2 &        11 &       4 &      3 &        6 &         9 &          2 &                     5 &     6 &     5 &              3 &      5 &     7 &        9 &           5 &      6 &           \cellcolor{blue!10}5.5 \\
\hline
\end{tabular}
\label{tab:rankings_recall_knn_10}
\end{table}

\subsection*{KMeans experiments}

\subsubsection*{Without noise}

\begin{table*}[!htbp]
\centering
\caption{Precision (\#P), Recall (\#R), and $\nu$ parameter on K-means algorithm without noise}
\vspace{3mm}
\makebox[\textwidth][c]{
\resizebox{\textwidth}{!}{
\begin{tabular}{|p{2.8cm}||p{1.5cm}|p{1cm}|p{1cm}||p{1.5cm}|p{1cm}|p{1cm}||p{1.5cm}|p{1cm}|p{1cm}||p{1.5cm}|p{1cm}|p{1cm}||p{1.5cm}|p{1cm}|p{1cm}||p{1.5cm}|p{1cm}|p{1cm}||p{1.5cm}|p{1cm}|p{1cm}||p{1.5cm}|p{1cm}|p{1cm}|}
\hline
\multicolumn{1}{|c||}{} & \multicolumn{3}{c||}{Ionosphere} & \multicolumn{3}{c||}{Banknote} & \multicolumn{3}{c||}{German} & \multicolumn{3}{c||}{Heart} & \multicolumn{3}{c||}{Balance} & \multicolumn{3}{c||}{Haberman} & \multicolumn{3}{c||}{Wholesale} & \multicolumn{3}{c||}{Indian Liver Patient} \\
\hline
Distance & $\nu$ & \#P & \#R & $\nu$ & \#P & \#R & $\nu$ & \#P & \#R & $\nu$ & \#P & \#R & $\nu$ & \#P & \#R & $\nu$ & \#P & \#R & $\nu$ & \#P & \#R & $\nu$ & \#P & \#R \\
\hline
Gini Generalized & 1.18 & 0.4915 & 0.4161 & 4.88 & \cellcolor{blue!30}0.6339 & \cellcolor{blue!30}0.6363 & 2.48 & \cellcolor{blue!30}0.6252 & 0.5650 & 1.72 & \cellcolor{blue!30}0.6186 & 0.3495 & 1.84 & \cellcolor{blue!30}0.8214 & 0.5520 & 2.14 & 0.4772 & 0.3409 & 2.36 & 0.4466 & 0.1273 & 1.62 & 0.7052 & 0.0759 \\
Gini & - & 0.3897 & 0.3535 & - & 0.5822 & 0.5816 & - & 0.4645 & 0.5360 & - & 0.4916 & 0.3330 & - & 0.7131 & 0.5408 & - & 0.4082 & 0.3475 & - & 0.1467 & 0.1273 & - & 0.6828 & 0.1174 \\
Euclidean & - & 0.3147 & 0.2882 & - & 0.6077 & 0.6115 & - & 0.4123 & 0.5130 & - & 0.4239 & 0.2672 & - & 0.6977 & 0.5296 & - & 0.3183 & 0.3213 & - & 0.5445 & 0.1227 & - & 0.4288 & 0.0155 \\
Manhattan & - & 0.3586 & 0.3280 & - & 0.6083 & 0.6123 & - & 0.4123 & 0.5130 & - & 0.4113 & 0.2573 & - & 0.7204 & \cellcolor{blue!30}0.5728 & - & 0.3247 & 0.3443 & - & \cellcolor{blue!30}0.5992 & 0.0977 & - & 0.4288 & 0.0155 \\
Minkowski & 3.0 & 0.3200 & 0.2910 & 3.0 & 0.6085 & 0.6123 & 3.0 & 0.4123 & 0.5130 & 3.0 & 0.4254 & 0.2605 & 3.0 & 0.6475 & 0.4928 & 3.0 & 0.3208 & 0.2909 & 3.0 & 0.5315 & 0.1386 & 3.0 & 0.4288 & 0.0155 \\
\hline
\end{tabular}}}
\label{res_noise_0}
\end{table*}

\begin{table*}[!htbp]
\centering
\caption{Precision (\#P), Recall (\#R), and $\nu$ parameter on K-means algorithm without noise}
\vspace{3mm}
\makebox[\textwidth][c]{
\resizebox{\textwidth}{!}{
\begin{tabular}{|p{2.8cm}||p{1.5cm}|p{1cm}|p{1cm}||p{1.5cm}|p{1cm}|p{1cm}||p{1.5cm}|p{1cm}|p{1cm}||p{1.5cm}|p{1cm}|p{1cm}||p{1.5cm}|p{1cm}|p{1cm}||p{1.5cm}|p{1cm}|p{1cm}||p{1.5cm}|p{1cm}|p{1cm}||p{1.5cm}|p{1cm}|p{1cm}||}
\hline
\multicolumn{1}{|c||}{} & \multicolumn{3}{c||}{Iris} & \multicolumn{3}{c||}{Wine} & \multicolumn{3}{c||}{Breast-Cancer} & \multicolumn{3}{c||}{Sonar} & \multicolumn{3}{c||}{QSAR} & \multicolumn{3}{c||}{Vehicle} & \multicolumn{3}{c||}{Australian} & \multicolumn{3}{c||}{Glass} \\
\hline
Distance & $k$ \hspace{0.5cm} $\nu$ & \#P & \#R & $k$ \hspace{0.5cm} $\nu$ & \#P & \#R & $k$ \hspace{0.5cm} $\nu$ & \#P & \#R & $k$ \hspace{0.5cm} $\nu$ & \#P & \#R & $k$ \hspace{0.5cm} $\nu$ & \#P & \#R & $k$ \hspace{0.5cm} $\nu$ & \#P & \#R & $k$ \hspace{0.5cm} $\nu$ & \#P & \#R & $k$ \hspace{0.5cm} $\nu$ & \#P & \#R \\
\hline
Gini Generalized & 3 \hspace{0.5cm} 3.52 & 0.0998 & 0.1 & 3 \hspace{0.5cm} 4.10 & 0.1913 & 0.1738 & 2 \hspace{0.5cm} 2.70 & 0.9234 & 0.9157 & 2 \hspace{0.5cm} 1.22 & \cellcolor{blue!30}0.6339 & \cellcolor{blue!30}0.5956 & 3 \hspace{0.5cm} 3.12 & 0.3741 & 0.2978 & 3 \hspace{0.5cm} 4.52 & 0.2946 & 0.2801 & 2 \hspace{0.5cm} 1.66 & 0.7693 & 0.6145 & 6 \hspace{0.5cm} 2.54 & 0.1492 & 0.0699 \\
Gini & 3 \hspace{0.5cm} 2 & 0.0828 & 0.1 & 3 \hspace{0.5cm} 2 & 0.1172 & 0.1005 & 2 \hspace{0.5cm} 2 & 0.9129 & 0.9052 & 2 \hspace{0.5cm} 2 & 0.5985 & 0.5618 & 3 \hspace{0.5cm} 2 & 0.3419 & 0.2644 & 3 \hspace{0.5cm} 2 & 0.2596 & 0.2517 & 2 \hspace{0.5cm} 2 & 0.6891 & 0.6101 & 6 \hspace{0.5cm} 2 & 0.0584 & 0.0327 \\
Euclidean & 3 \hspace{0.5cm} - & 0.0757 & 0.0933 & 3 \hspace{0.5cm} - & 0.0610 & 0.0502 & 2 \hspace{0.5cm} - & 0.8802 & 0.8524 & 2 \hspace{0.5cm} - & 0.5636 & 0.5373 & 3 \hspace{0.5cm} - & 0.3520 & 0.2349 & 3 \hspace{0.5cm} - & 0.2088 & 0.2269 & 2 \hspace{0.5cm} - & 0.4873 & 0.5590 & 6 \hspace{0.5cm} - & 0.0388 & 0.0186 \\
\hline
\end{tabular}}}
\label{res_nois_0_bis}
\end{table*}

\begin{table}[!htbp]
\centering
\caption{Ranking of KMeans Models by Precision}
\vspace{3mm}
\tiny
\renewcommand{\arraystretch}{1} 
\setlength{\tabcolsep}{1.5pt} 
\begin{tabular}{|p{0.9in}|c|c|c|c|c|c|c|c|c|c|c|c|c|c|c|c|c|}
\hline
\textbf{Distances}   &  \textbf{I} & \textbf{B} & \textbf{G} & \textbf{He} & \textbf{Ba} & \textbf{Ha} & \textbf{W} & \textbf{In} &  \textbf{Ir} & \textbf{Wi} & \textbf{Bc} & \textbf{S} & \textbf{QS} & \textbf{V} & \textbf{A} & \textbf{Gl} &  \textbf{Rank} \\ \hline
Gini Generalized &          4 &        1 &      1 &     1 &       1 &        5 &         6 &                    3 &    4 &    6 &             3 &     1 &    8 &       6 &          2 &     4 &         \cellcolor{blue!40}3.50 \\
Gini             &          7 &        8 &      4 &     5 &       4 &        7 &        12 &                    6 &    5 &    8 &             4 &     2 &   12 &       7 &          3 &     8 &         6.38 \\
Euclidean        &         10 &        4 &      5 &     8 &       5 &       11 &         3 &                   10 &    8 &   11 &             8 &     7 &   10 &      10 &         10 &    10 &         8.12 \\
Manhattan        &          8 &        3 &      5 &    10 &       3 &        9 &         1 &                   10 &    6 &    9 &             9 &     5 &    9 &      11 &         10 &     6 &         7.12 \\
Minkowski        &          9 &        2 &      5 &     7 &       6 &       10 &         5 &                   10 &    7 &   10 &             7 &     6 &   11 &       9 &         10 &    11 &         7.81 \\
Cosine           &          6 &        6 &     10 &     3 &       2 &        8 &         2 &                    8 &   12 &    7 &             5 &     8 &    6 &      12 &          5 &     7 &         6.69 \\
Canberra         &          2 &        7 &      3 &     2 &       9 &        1 &        11 &                    5 &   11 &    2 &             2 &    10 &    4 &       8 &          9 &     5 &         5.69 \\
\hline
\end{tabular}
\label{tab:rankings_precision_kmeans_2}
\end{table}

\newpage

\begin{table}[!htbp]
\centering
\caption{Ranking of KMeans Models by Recall}
\vspace{3mm}
\tiny
\renewcommand{\arraystretch}{1} 
\setlength{\tabcolsep}{1.5pt} 
\begin{tabular}{|p{0.9in}|c|c|c|c|c|c|c|c|c|c|c|c|c|c|c|c|c|}
\hline
\textbf{Distances}   &  \textbf{I} & \textbf{B} & \textbf{G} & \textbf{He} & \textbf{Ba} & \textbf{Ha} & \textbf{W} & \textbf{In} &  \textbf{Ir} & \textbf{Wi} & \textbf{Bc} & \textbf{S} & \textbf{QS} & \textbf{V} & \textbf{A} & \textbf{Gl} &  \textbf{Rank} \\ \hline
Gini Generalized &          4 &        1 &      2 &     2 &       2 &        9 &         7 &                    7 &    4 &    6 &             2 &     1 &    6 &       6 &          4 &     4 &         \cellcolor{blue!40}4.19 \\
Gini             &          9 &        8 &      3 &     3 &       3 &        7 &         7 &                    6 &    4 &    9 &             4 &     3 &    8 &       7 &          5 &     6 &         5.75 \\
Euclidean        &         12 &        4 &      5 &     5 &       4 &       10 &        10 &                   10 &    7 &   11 &             8 &     8 &   11 &       9 &          9 &    10 &         8.31 \\
Manhattan        &         10 &        2 &      5 &     8 &       1 &        8 &        12 &                   10 &    4 &   10 &             9 &     5 &    9 &      10 &          9 &     8 &         7.5 \\
Minkowski        &         11 &        2 &      5 &     7 &       6 &       11 &         6 &                   10 &    7 &   11 &             7 &     7 &   12 &       8 &          9 &    10 &         8.06 \\
Cosine           &          6 &        5 &     11 &     4 &       5 &        6 &         2 &                    4 &   12 &    8 &             5 &     9 &    4 &      12 &          2 &     7 &         6.38 \\
Canberra         &          3 &        6 &      4 &     9 &      10 &        3 &        11 &                    2 &   10 &    3 &             1 &    11 &    2 &      10 &         12 &     5 &         6.38 \\
\hline
\end{tabular}
\label{tab:rankings_recall_kmeans_2}
\end{table}

\subsubsection*{With noise at level 5\%}

\begin{table*}[!htbp]
\centering
\caption{Precision (\#P), Recall (\#R), and $\nu$ parameter on KMeans algorithm for noise level 5\%.}
\vspace{3mm}
\makebox[\textwidth][c]{
\resizebox{\textwidth}{!}{
\begin{tabular}{|p{2.8cm}||p{1.5cm}|p{1cm}|p{1cm}||p{1.5cm}|p{1cm}|p{1cm}||p{1.5cm}|p{1cm}|p{1cm}||p{1.5cm}|p{1cm}|p{1cm}||p{1.5cm}|p{1cm}|p{1cm}||p{1.5cm}|p{1cm}|p{1cm}||p{1.5cm}|p{1cm}|p{1cm}||p{1.5cm}|p{1cm}|p{1cm}||}
\hline
\multicolumn{1}{|c||}{} & \multicolumn{3}{c||}{Ionosphere} & \multicolumn{3}{c||}{Banknote} & \multicolumn{3}{c||}{German} & \multicolumn{3}{c||}{Heart} & \multicolumn{3}{c||}{Balance} & \multicolumn{3}{c||}{Haberman} & \multicolumn{3}{c||}{Wholesale} & \multicolumn{3}{c||}{Indian Liver Patient} \\
\hline
Distance & $\nu$ & \#P & \#R & $\nu$ & \#P & \#R & $\nu$ & \#P & \#R & $\nu$ & \#P & \#R & $\nu$ & \#P & \#R & $\nu$ & \#P & \#R & $\nu$ & \#P & \#R & $\nu$ & \#P & \#R \\
\hline
Gini\_generalized & 1.1 & 0.5391 & 0.4359 & 4.92 & \cellcolor{blue!30}0.6306 & \cellcolor{blue!30}0.6334 & 2.40 & 0.5501 & 0.5470 & 2.32 & \cellcolor{blue!30}0.5959 & \cellcolor{blue!30}0.3566 & 2.44 & 0.6283 & 0.4736 & 1.82 & 0.5561 & 0.3499 & 2.30 & 0.2688 & 0.1364 & 1.58 & 0.7066 & 0.0552 \\
Gini & - & 0.3996 & 0.3447 & - & 0.5779 & 0.5773 & - & 0.5140 & 0.5140 & - & 0.5373 & 0.3633 & - & 0.5660 & 0.4336 & - & 0.5344 & 0.3432 & - & 0.1457 & 0.1341 & - & 0.6802 & 0.1054 \\
Euclidean & - & 0.3344 & 0.2906 & - & 0.6070 & 0.6115 & - & 0.5120 & 0.5120 & - & 0.5136 & 0.2574 & - & 0.5975 & 0.4576 & - & 0.5479 & 0.3335 & - & 0.5428 & 0.1250 & - & \cellcolor{blue!30}0.7150 & 0.0155 \\
Manhattan & - & 0.3792 & 0.3276 & - & 0.6071 & 0.6115 & - & 0.5110 & 0.5110 & - & 0.4994 & 0.2540 & - & \cellcolor{blue!30}0.6476 & \cellcolor{blue!30}0.4912 & - & 0.5332 & 0.3335 & - & 0.5508 & 0.0977 & - & \cellcolor{blue!30}0.7150 & 0.0155 \\
Minkowski & 3.0 & 0.3309 & 0.2878 & 3.0 & 0.6072 & 0.6115 & 3.0 & 0.5120 & 0.5120 & 3.0 & 0.5045 & 0.2607 & 3.0 & 0.6127 & 0.4752 & 3.0 & 0.5515 & 0.3204 & 3.0 & 0.5340 & 0.1409 & 3.0 & \cellcolor{blue!30}0.7150 & 0.0155 \\
\hline
\end{tabular}}}
\label{res_noise_5}
\end{table*}

\begin{table*}[!htbp]
\centering
\caption{Precision (\#P), Recall (\#R), and $\nu$ parameter on KMeans algorithm for noise level 5\%.}
\vspace{3mm}
\makebox[\textwidth][c]{
\resizebox{\textwidth}{!}{
\begin{tabular}{|p{2.8cm}||p{1.5cm}|p{1cm}|p{1cm}||p{1.5cm}|p{1cm}|p{1cm}||p{1.5cm}|p{1cm}|p{1cm}||p{1.5cm}|p{1cm}|p{1cm}||p{1.5cm}|p{1cm}|p{1cm}||p{1.5cm}|p{1cm}|p{1cm}||p{1.5cm}|p{1cm}|p{1cm}||p{1.5cm}|p{1cm}|p{1cm}|}
\hline
\multicolumn{1}{|c||}{} & \multicolumn{3}{c||}{Iris} & \multicolumn{3}{c||}{Wine} & \multicolumn{3}{c||}{Breast-Cancer} & \multicolumn{3}{c||}{Sonar} & \multicolumn{3}{c||}{QSAR} & \multicolumn{3}{c||}{Vehicle} & \multicolumn{3}{c||}{Australian} & \multicolumn{3}{c||}{Glass} \\
\hline
Distance & $\nu$ & \#P & \#R & $\nu$ & \#P & \#R & $\nu$ & \#P & \#R & $\nu$ & \#P & \#R & $\nu$ & \#P & \#R & $\nu$ & \#P & \#R & $\nu$ & \#P & \#R & $\nu$ & \#P & \#R \\
\hline
Gini\_generalized & 2.68 & 0.0939 & 0.1067 & 4.56 & 0.2085 & 0.2021 & 2.24 & \cellcolor{blue!30}0.9283 & \cellcolor{blue!30}0.9245 & 2.04 & 0.5750 & 0.5722 & 2.30 & 0.3871 & 0.2850 & 4.76 & 0.3050 & 0.2872 & 1.40 & 0.7677 & 0.6000 & 1.50 & 0.0918 & 0.0514 \\
Gini & - & 0.0761 & 0.1000 & - & 0.1207 & 0.1010 & - & 0.9107 & 0.9034 & - & 0.5460 & 0.5431 & - & 0.3542 & 0.2684 & - & 0.2654 & 0.2576 & - & 0.7086 & 0.6145 & - & 0.0774 & 0.0421 \\
Euclidean & - & 0.0728 & 0.0933 & - & 0.0610 & 0.0449 & - & 0.8771 & 0.8489 & - & 0.5448 & 0.5382 & - & 0.3563 & 0.2337 & - & 0.2105 & 0.2292 & - & 0.4877 & 0.5594 & - & 0.0358 & 0.0188 \\
Manhattan & - & 0.0761 & 0.1000 & - & 0.0554 & 0.0394 & - & 0.8735 & 0.8437 & - & 0.5621 & 0.5575 & - & 0.3669 & 0.2632 & - & 0.2055 & 0.2221 & - & 0.4877 & 0.5594 & - & 0.0455 & 0.0235 \\
Minkowski & 3.0 & 0.0690 & 0.0867 & 3.0 & 0.0610 & 0.0449 & 3.0 & 0.8793 & 0.8524 & 3.0 & 0.5331 & 0.5240 & 3.0 & 0.3552 & 0.2325 & 3.0 & 0.2098 & 0.2292 & 3.0 & 0.4877 & 0.5594 & 3.0 & 0.0336 & 0.0188 \\
Cosine & - & 0.0211 & 0.0200 & - & 0.1536 & 0.1235 & - & 0.9131 & 0.9034 & - & 0.5460 & 0.5431 & - & 0.3990 & 0.3081 & - & 0.1968 & 0.2162 & - & 0.6776 & 0.6725 & - & 0.0554 & 0.0467 \\
\hline
\end{tabular}}}
\label{res_noise_5_restants}
\end{table*}

\begin{table}[!htbp]
\centering
\caption{Ranking of KMeans Models by Precision: Noise 5\%}
\vspace{3mm}
\tiny
\renewcommand{\arraystretch}{1} 
\setlength{\tabcolsep}{1.5pt} 
\begin{tabular}{|p{0.9in}|c|c|c|c|c|c|c|c|c|c|c|c|c|c|c|c|c|}
\hline
\textbf{Distances}   &  \textbf{I} & \textbf{B} & \textbf{G} & \textbf{He} & \textbf{Ba} & \textbf{Ha} & \textbf{W} & \textbf{In} &  \textbf{Ir} & \textbf{Wi} & \textbf{Bc} & \textbf{S} & \textbf{QS} & \textbf{V} & \textbf{A} & \textbf{Gl} &  \textbf{Rank} \\ \hline
Gini Generalized &           4 &         1 &       2 &      1 &        3 &         2 &          9 &                     6 &     4 &     6 &              1 &      3 &     7 &        5 &           2 &      4 &           \cellcolor{blue!40}3.75 \\
Gini             &           8 &         8 &       4 &      3 &        8 &         8 &         12 &                     9 &     5 &     8 &              5 &      6 &    11 &        6 &           4 &      6 &           6.94 \\
Euclidean        &          10 &         4 &       5 &      5 &        6 &         5 &          3 &                     1 &     7 &     9 &              8 &      8 &     9 &        9 &           8 &      9 &           6.62 \\
Manhattan        &           9 &         3 &       7 &      7 &        1 &         9 &          2 &                     1 &     5 &    11 &              9 &      4 &     8 &       11 &           8 &      8 &           6.44 \\
Minkowski        &          11 &         2 &       5 &      6 &        4 &         4 &          4 &                     1 &     8 &     9 &              7 &      9 &    10 &       10 &           8 &     11 &           6.81 \\
Cosine           &           7 &         6 &      11 &      4 &        2 &         6 &          1 &                    10 &    12 &     7 &              4 &      6 &     5 &       12 &           5 &      7 &           6.56 \\
Canberra         &           3 &         5 &       3 &      8 &        5 &         3 &         11 &                     8 &    11 &     4 &              3 &      2 &     1 &        8 &           7 &      5 &           \cellcolor{blue!10}5.44 \\
\hline
\end{tabular}
\label{tab:rankings_recall_kmeans_5}
\end{table}

\begin{table}[!htbp]
\centering
\caption{Ranking of KMeans Models by Recall: Noise 5\%}
\vspace{3mm}
\tiny
\renewcommand{\arraystretch}{1} 
\setlength{\tabcolsep}{1.5pt} 
\begin{tabular}{|p{0.9in}|c|c|c|c|c|c|c|c|c|c|c|c|c|c|c|c|c|}
\hline
\textbf{Distances}   &  \textbf{I} & \textbf{B} & \textbf{G} & \textbf{He} & \textbf{Ba} & \textbf{Ha} & \textbf{W} & \textbf{In} &  \textbf{Ir} & \textbf{Wi} & \textbf{Bc} & \textbf{S} & \textbf{QS} & \textbf{V} & \textbf{A} & \textbf{Gl} &  \textbf{Rank} \\ \hline
Gini Generalized &           5 &         1 &       2 &      2 &        4 &         5 &          8 &                     7 &     4 &     6 &              1 &      2 &     6 &        5 &           4 &      4 &          \cellcolor{blue!40}4.12 \\
Gini             &           9 &         8 &       4 &      1 &        8 &         6 &          9 &                     6 &     5 &     9 &              4 &      5 &     7 &        6 &           3 &      7 &          6.06 \\
Euclidean        &          11 &         2 &       5 &      6 &        6 &         7 &         10 &                    10 &     7 &    10 &              8 &      7 &    11 &        9 &           9 &     10 &          8 \\
Manhattan        &          10 &         2 &       7 &      7 &        1 &         7 &         12 &                    10 &     5 &    12 &              9 &      3 &     8 &       11 &           9 &      8 &          7.56 \\
Minkowski        &          12 &         2 &       5 &      5 &        3 &         9 &          6 &                    10 &     8 &    10 &              7 &     10 &    12 &        9 &           9 &     10 &          7.94 \\
Cosine           &           7 &         6 &      11 &      4 &        2 &         4 &          1 &                     4 &    12 &     7 &              4 &      5 &     5 &       12 &           2 &      6 &          5.75 \\
Canberra         &           4 &         5 &       3 &     11 &        5 &         3 &         11 &                     3 &    10 &     5 &              2 &      1 &     2 &        8 &           5 &      5 &          5.19 \\
\hline
\end{tabular}
\label{tab:rankings_recall_kmeans_5_2}
\end{table}

\newpage

\subsubsection*{With noise at level 10\%}

\begin{table*}[!htbp]
\centering
\caption{Precision (\#P), Recall (\#R), and $\nu$ parameter on KMeans algorithm for noise level 10\%.}
\vspace{3mm}
\makebox[\textwidth][c]{
\resizebox{\textwidth}{!}{
\begin{tabular}{|p{2.8cm}||p{1.5cm}|p{1cm}|p{1cm}||p{1.5cm}|p{1cm}|p{1cm}||p{1.5cm}|p{1cm}|p{1cm}||p{1.5cm}|p{1cm}|p{1cm}||p{1.5cm}|p{1cm}|p{1cm}||p{1.5cm}|p{1cm}|p{1cm}||p{1.5cm}|p{1cm}|p{1cm}||p{1.5cm}|p{1cm}|p{1cm}|}
\hline
\multicolumn{1}{|c||}{} & \multicolumn{3}{c||}{Iris} & \multicolumn{3}{c||}{Wine} & \multicolumn{3}{c||}{Breast-Cancer} & \multicolumn{3}{c||}{Sonar} & \multicolumn{3}{c||}{QSAR} & \multicolumn{3}{c||}{Vehicle} & \multicolumn{3}{c||}{Australian} & \multicolumn{3}{c||}{Glass} \\
\hline
Distance & $\nu$ & \#P & \#R & $\nu$ & \#P & \#R & $\nu$ & \#P & \#R & $\nu$ & \#P & \#R & $\nu$ & \#P & \#R & $\nu$ & \#P & \#R & $\nu$ & \#P & \#R & $\nu$ & \#P & \#R \\
\hline
Gini\_generalized & 3.38 & 0.0946 & 0.1067 & 4.78 & 0.2116 & 0.2022 & 2.24 & \cellcolor{blue!30}0.9283 & \cellcolor{blue!30}0.9245 & 2.94 & 0.5864 & \cellcolor{blue!30}0.5815 & 2.14 & 0.3871 & 0.2863 & 4.44 & 0.3080 & 0.2872 & 1.34 & 0.7672 & 0.5986 & 2.38 & 0.1010 & 0.0560 \\
Gini & - & 0.0768 & 0.1000 & - & 0.1207 & 0.1010 & - & 0.9107 & 0.9034 & - & 0.5586 & 0.5527 & - & 0.3614 & 0.2722 & - & 0.2651 & 0.2576 & - & 0.7101 & 0.6159 & - & 0.0356 & 0.0186 \\
Euclidean & - & 0.0806 & 0.1000 & - & 0.0610 & 0.0449 & - & 0.8771 & 0.8489 & - & 0.5536 & 0.5479 & - & 0.3571 & 0.2350 & - & 0.2097 & 0.2281 & - & 0.4877 & 0.5594 & - & 0.0435 & 0.0235 \\
Manhattan & - & 0.0761 & 0.1000 & - & 0.0554 & 0.0394 & - & 0.8735 & 0.8437 & - & 0.5488 & 0.5432 & - & 0.3640 & 0.2581 & - & 0.2055 & 0.2221 & - & 0.4877 & 0.5594 & - & 0.0502 & 0.0281 \\
Minkowski & 3.0 & 0.0751 & 0.0933 & 3.0 & 0.0610 & 0.0449 & 3.0 & 0.8793 & 0.8524 & 3.0 & 0.5742 & 0.5672 & 3.0 & 0.3551 & 0.2325 & 3.0 & 0.2105 & 0.2304 & 3.0 & 0.4877 & 0.5594 & 3.0 & 0.0435 & 0.0235 \\
Cosine & - & 0.0216 & 0.0200 & - & 0.1514 & 0.1235 & - & 0.9131 & 0.9034 & - & 0.5523 & 0.5479 & - & 0.3989 & 0.3056 & - & 0.1968 & 0.2162 & - & 0.6776 & 0.6725 & - & 0.0626 & 0.0467 \\
\hline
\end{tabular}}}
\label{res_noise_10_restants}
\end{table*}

\begin{table*}[!htbp]
\centering
\caption{Precision (\#P), Recall (\#R), and $\nu$ parameter on KMeans algorithm for noise level 10\%.}
\vspace{3mm}
\makebox[\textwidth][c]{
\resizebox{\textwidth}{!}{
\begin{tabular}{|p{2.8cm}||p{1.5cm}|p{1cm}|p{1cm}||p{1.5cm}|p{1cm}|p{1cm}||p{1.5cm}|p{1cm}|p{1cm}||p{1.5cm}|p{1cm}|p{1cm}||p{1.5cm}|p{1cm}|p{1cm}||p{1.5cm}|p{1cm}|p{1cm}||p{1.5cm}|p{1cm}|p{1cm}||p{1.5cm}|p{1cm}|p{1cm}||}
\hline
\multicolumn{1}{|c||}{} & \multicolumn{3}{c||}{Ionosphere} & \multicolumn{3}{c||}{Banknote} & \multicolumn{3}{c||}{German} & \multicolumn{3}{c||}{Heart} & \multicolumn{3}{c||}{Balance} & \multicolumn{3}{c||}{Haberman} & \multicolumn{3}{c||}{Wholesale} & \multicolumn{3}{c||}{Indian Liver Patient} \\
\hline
Distance & $\nu$ & \#P & \#R & $\nu$ & \#P & \#R & $\nu$ & \#P & \#R & $\nu$ & \#P & \#R & $\nu$ & \#P & \#R & $\nu$ & \#P & \#R & $\nu$ & \#P & \#R & $\nu$ & \#P & \#R \\
\hline
Gini\_generalized & 1.1 & 0.4815 & 0.4045 & 4.98 & \cellcolor{blue!30}0.6305 & \cellcolor{blue!30}0.6334 & 2.38 & 0.5510 & 0.5480 & 2.32 & 0.5966 & 0.3532 & 2.16 & 0.6248 & 0.4608 & 1.74 & 0.5615 & 0.3435 & 2.30 & 0.2688 & 0.1364 & 1.56 & 0.7066 & 0.0552 \\
Gini & - & 0.3924 & 0.3389 & - & 0.5787 & 0.5780 & - & 0.5140 & 0.5140 & - & 0.5389 & 0.3601 & - & 0.5608 & 0.4336 & - & 0.5451 & 0.3433 & - & 0.1457 & 0.1341 & - & 0.6802 & 0.1054 \\
Euclidean & - & 0.3351 & 0.2905 & - & 0.6032 & 0.6079 & - & 0.5120 & 0.5120 & - & 0.5136 & 0.2574 & - & 0.6109 & 0.4688 & - & 0.5425 & 0.3302 & - & 0.5428 & 0.1250 & - & \cellcolor{blue!30}0.7150 & 0.0155 \\
Manhattan & - & 0.3825 & 0.3304 & - & 0.6081 & 0.6123 & - & 0.5120 & 0.5120 & - & 0.5040 & 0.2540 & - & 0.6367 & 0.4768 & - & 0.5332 & 0.3335 & - & 0.5508 & 0.0977 & - &\cellcolor{blue!30} 0.7150 & 0.0155 \\
Minkowski & 3.0 & 0.3342 & 0.2906 & 3.0 & 0.6026 & 0.6072 & 3.0 & 0.5120 & 0.5120 & 3.0 & 0.5018 & 0.2607 & 3.0 & 0.6067 & \cellcolor{blue!30}0.4864 & 3.0 & 0.5529 & 0.3237 & 3.0 & 0.5340 & 0.1409 & 3.0 & \cellcolor{blue!30}0.7150 & 0.0155 \\
\hline
\end{tabular}}}
\label{res_noise_10}
\end{table*}

\begin{table}[!htbp]
\centering
\caption{Ranking of KMeans Models by Precision: Noise 10\%}
\vspace{3mm}
\tiny
\renewcommand{\arraystretch}{1} 
\setlength{\tabcolsep}{1.5pt} 
\begin{tabular}{|p{0.9in}|c|c|c|c|c|c|c|c|c|c|c|c|c|c|c|c|c|}
\hline
\textbf{Distances}   &  \textbf{I} & \textbf{B} & \textbf{G} & \textbf{He} & \textbf{Ba} & \textbf{Ha} & \textbf{W} & \textbf{In} &  \textbf{Ir} & \textbf{Wi} & \textbf{Bc} & \textbf{S} & \textbf{QS} & \textbf{V} & \textbf{A} & \textbf{Gl} &  \textbf{Rank} \\ \hline
Gini Generalized &           3 &         1 &       2 &      4 &        4 &         3 &         10 &                     6 &     4 &     6 &              1 &      2 &     8 &        6 &           2 &      5 &           \cellcolor{blue!40}4.19 \\
Gini             &           8 &         8 &       4 &      5 &        9 &         6 &         12 &                     9 &     6 &     8 &              5 &      4 &    10 &        7 &           4 &     11 &           7.25 \\
Euclidean        &          10 &         3 &       5 &      7 &        6 &         8 &          3 &                     1 &     5 &     9 &              8 &      5 &    11 &       10 &           8 &      9 &           6.75 \\
Manhattan        &           9 &         2 &       5 &      9 &        2 &         9 &          2 &                     1 &     7 &    11 &              9 &      7 &     9 &       11 &           8 &      8 &           6.81 \\
Minkowski        &          11 &         4 &       5 &     10 &        7 &         4 &          4 &                     1 &     8 &     9 &              7 &      3 &    12 &        9 &           8 &      9 &           6.94 \\
Cosine           &           6 &         5 &      10 &      6 &        3 &         5 &          1 &                    10 &    11 &     7 &              4 &      6 &     6 &       12 &           5 &      7 &           6.5 \\
Canberra         &           4 &         6 &       1 &      1 &        8 &         7 &         11 &                     8 &    10 &     3 &              3 &      9 &     3 &        8 &           7 &      4 &           5.81 \\
\hline
\end{tabular}
\label{tab:rankings_precision_kmeans_10}
\end{table}

\begin{table}[!htbp]
\centering
\caption{Ranking of KMeans Models by Recall: Noise 10\%}
\vspace{3mm}
\tiny
\renewcommand{\arraystretch}{1} 
\setlength{\tabcolsep}{1.5pt} 
\begin{tabular}{|p{0.9in}|c|c|c|c|c|c|c|c|c|c|c|c|c|c|c|c|c|}
\hline
\textbf{Distances}   &  \textbf{I} & \textbf{B} & \textbf{G} & \textbf{He} & \textbf{Ba} & \textbf{Ha} & \textbf{W} & \textbf{In} &  \textbf{Ir} & \textbf{Wi} & \textbf{Bc} & \textbf{S} & \textbf{QS} & \textbf{V} & \textbf{A} & \textbf{Gl} &  \textbf{Rank} \\ \hline
Gini Generalized &           3 &         1 &       2 &      3 &        6 &         6 &          8 &                     7 &     4 &     6 &              1 &      1 &     7 &        6 &           4 &      5 &           \cellcolor{blue!40}4.38 \\
Gini             &           9 &         8 &       4 &      2 &        8 &         7 &          9 &                     6 &     5 &     9 &              4 &      3 &     8 &        7 &           3 &     11 &           6.44 \\
Euclidean        &          12 &         3 &       5 &      7 &        5 &        10 &         10 &                    10 &     5 &    10 &              8 &      5 &    11 &        9 &           9 &      9 &           8 \\
Manhattan        &          10 &         2 &       5 &      8 &        4 &         9 &         12 &                    10 &     5 &    12 &              9 &      7 &     9 &       11 &           9 &      8 &           8.12 \\
Minkowski        &          11 &         4 &       5 &      6 &        1 &        11 &          7 &                    10 &     8 &    10 &              7 &      2 &    12 &        8 &           9 &      9 &           7.5 \\
Cosine           &           6 &         6 &      11 &      5 &        2 &         4 &          1 &                     4 &    11 &     8 &              4 &      5 &     5 &       12 &           2 &      6 &           5.75 \\
Canberra         &           5 &         5 &       1 &      4 &        7 &         3 &         11 &                     3 &    10 &     2 &              2 &      4 &     3 &       10 &           7 &      4 &           \cellcolor{blue!20}5.06 \\
\hline
\end{tabular}
\label{tab:rankings_recall_kmeans_10}
\end{table}

\newpage

\subsubsection*{Real application: maximizing the silhouette score}

\begin{table*}[!htbp]
\centering
\caption{Precision (\#P), Recall (\#R), $k$ number of classes and optimal $\nu$ parameter for Gini generalized prametric and Gini generalized silhouette on K-means algorithm for various distances on UCI datasets without noise.}
\vspace{3mm}
\makebox[\textwidth][c]{
\resizebox{\textwidth}{!}{
\begin{tabular}{|p{2.8cm}||p{1.5cm}|p{1cm}|p{1cm}||p{1.5cm}|p{1cm}|p{1cm}||p{1.5cm}|p{1cm}|p{1cm}||p{1.5cm}|p{1cm}|p{1cm}||p{1.5cm}|p{1cm}|p{1cm}||p{1.5cm}|p{1cm}|p{1cm}||p{1.5cm}|p{1cm}|p{1cm}||p{1.5cm}|p{1cm}|p{1cm}|}
\hline
\multicolumn{1}{|c||}{} & \multicolumn{3}{c||}{Ionosphere} & \multicolumn{3}{c||}{Banknote} & \multicolumn{3}{c||}{German} & \multicolumn{3}{c||}{Heart} & \multicolumn{3}{c||}{Balance} & \multicolumn{3}{c||}{Haberman} & \multicolumn{3}{c||}{Wholesale} & \multicolumn{3}{c||}{Indian Liver Patient} \\
\hline
Distance & $\nu$ & \#P & \#R & $\nu$ & \#P & \#R & $\nu$ & \#P & \#R & $\nu$ & \#P & \#R & $\nu$ & \#P & \#R & $\nu$ & \#P & \#R & $\nu$ & \#P & \#R & $\nu$ & \#P & \#R \\
\hline
Gini Generalized & 1.18 & 0.4915 & 0.4161 & 4.88 & \cellcolor{blue!30}0.6339 & \cellcolor{blue!30}0.6363 & 2.48 & \cellcolor{blue!30}0.6252 & 0.5650 & 1.72 & \cellcolor{blue!30}0.6186 & 0.3495 & 1.84 & \cellcolor{blue!30}0.8214 & 0.5520 & 2.14 & 0.4772 & 0.3409 & 2.36 & 0.4466 & 0.1273 & 1.62 & 0.7052 & 0.0759 \\
Gini G. silhouette & 4.62 & 0.3410 & 0.3080 & 4.60 & 0.6247 & 0.6275 & 2.20 & 0.4610 & 0.5560 & 1.76 & 0.5581 & 0.3132 & 1.88 & 0.6273 & 0.5312 & 2.34 & 0.4198 & 0.3443 & 2.82 & 0.2689 & 0.0295 & 1.10 & 0.6991 & 0.0432 \\
Gini & - & 0.3897 & 0.3535 & - & 0.5822 & 0.5816 & - & 0.4645 & 0.5360 & - & 0.4916 & 0.3330 & - & 0.7131 & 0.5408 & - & 0.4082 & 0.3475 & - & 0.1467 & 0.1273 & - & 0.6828 & 0.1174 \\
Euclidean & - & 0.3147 & 0.2882 & - & 0.6077 & 0.6115 & - & 0.4123 & 0.5130 & - & 0.4239 & 0.2672 & - & 0.6977 & 0.5296 & - & 0.3183 & 0.3213 & - & 0.5445 & 0.1227 & - & 0.4288 & 0.0155 \\
Manhattan & - & 0.3586 & 0.3280 & - & 0.6083 & 0.6123 & - & 0.4123 & 0.5130 & - & 0.4113 & 0.2573 & - & 0.7204 & \cellcolor{blue!30}0.5728 & - & 0.3247 & 0.3443 & - & \cellcolor{blue!30}0.5992 & 0.0977 & - & 0.4288 & 0.0155 \\
\hline
\end{tabular}}}
\label{res_noise_0_2}
\end{table*}

\begin{table*}[!htbp]
\centering
\caption{Precision (\#P), Recall (\#R), $k$ number of classes and optimal $\nu$ parameter for Gini generalized prametric and Gini generalized silhouette on K-means algorithm for various distances on UCI datasets without noise.}
\vspace{3mm}
\makebox[\textwidth][c]{
\resizebox{\textwidth}{!}{
\begin{tabular}{|p{2.8cm}||p{1.5cm}|p{1cm}|p{1cm}||p{1.5cm}|p{1cm}|p{1cm}||p{1.5cm}|p{1cm}|p{1cm}||p{1.5cm}|p{1cm}|p{1cm}||p{1.5cm}|p{1cm}|p{1cm}||p{1.5cm}|p{1cm}|p{1cm}||p{1.5cm}|p{1cm}|p{1cm}||p{1.5cm}|p{1cm}|p{1cm}||}
\hline
\multicolumn{1}{|c||}{} & \multicolumn{3}{c||}{Iris} & \multicolumn{3}{c||}{Wine} & \multicolumn{3}{c||}{Breast-Cancer} & \multicolumn{3}{c||}{Sonar} & \multicolumn{3}{c||}{QSAR} & \multicolumn{3}{c||}{Vehicle} & \multicolumn{3}{c||}{Australian} & \multicolumn{3}{c||}{Glass} \\
\hline
Distance & $k$ \hspace{0.5cm} $\nu$ & \#P & \#R & $k$ \hspace{0.5cm} $\nu$ & \#P & \#R & $k$ \hspace{0.5cm} $\nu$ & \#P & \#R & $k$ \hspace{0.5cm} $\nu$ & \#P & \#R & $k$ \hspace{0.5cm} $\nu$ & \#P & \#R & $k$ \hspace{0.5cm} $\nu$ & \#P & \#R & $k$ \hspace{0.5cm} $\nu$ & \#P & \#R & $k$ \hspace{0.5cm} $\nu$ & \#P & \#R \\
\hline
Gini Generalized & 3 \hspace{0.5cm} 3.52 & 0.0998 & 0.1 & 3 \hspace{0.5cm} 4.10 & 0.1913 & 0.1738 & 2 \hspace{0.5cm} 2.70 & 0.9234 & 0.9157 & 2 \hspace{0.5cm} 1.22 & \cellcolor{blue!30}0.6339 & \cellcolor{blue!30}0.5956 & 3 \hspace{0.5cm} 3.12 & 0.3741 & 0.2978 & 3 \hspace{0.5cm} 4.52 & 0.2946 & 0.2801 & 2 \hspace{0.5cm} 1.66 & 0.7693 & 0.6145 & 6 \hspace{0.5cm} 2.54 & 0.1492 & 0.0699 \\
Gini G. Silhouette & 3 \hspace{0.5cm} 1.40 & 0.08 & 0.1 & 3 \hspace{0.5cm} 1.72 & 0.0902 & 0.0890 & 2 \hspace{0.5cm} 1.12 & 0.8840 & 0.8577 & 2 \hspace{0.5cm} 2.38 & 0.5746 & 0.5423 & 3 \hspace{0.5cm} 1.88 & 0.3485 & 0.2670 & 3 \hspace{0.5cm} 1.10 & 0.2207 & 0.2186 & 2 \hspace{0.5cm} 1.10 & 0.7609 & 0.5797 & 6 \hspace{0.5cm} 1.14 & 0.0798 & 0.0420 \\
Gini & 3 \hspace{0.5cm} 2 & 0.0828 & 0.1 & 3 \hspace{0.5cm} 2 & 0.1172 & 0.1005 & 2 \hspace{0.5cm} 2 & 0.9129 & 0.9052 & 2 \hspace{0.5cm} 2 & 0.5985 & 0.5618 & 3 \hspace{0.5cm} 2 & 0.3419 & 0.2644 & 3 \hspace{0.5cm} 2 & 0.2596 & 0.2517 & 2 \hspace{0.5cm} 2 & 0.6891 & 0.6101 & 6 \hspace{0.5cm} 2 & 0.0584 & 0.0327 \\
\hline
\end{tabular}}}
\label{res_kmeans_gini_silhouette}
\end{table*}

\begin{table}[!htbp]
\centering
\caption{Ranking of KMeans Models by Precision: Maximizing the Silhouette Score}
\vspace{3mm}
\tiny
\renewcommand{\arraystretch}{1} 
\setlength{\tabcolsep}{1.5pt} 
\begin{tabular}{|p{0.9in}|c|c|c|c|c|c|c|c|c|c|c|c|c|c|c|c|c|}
\hline
\textbf{Distances}   &  \textbf{I} & \textbf{B} & \textbf{G} & \textbf{He} & \textbf{Ba} & \textbf{Ha} & \textbf{W} & \textbf{In} &  \textbf{Ir} & \textbf{Wi} & \textbf{Bc} & \textbf{S} & \textbf{QS} & \textbf{V} & \textbf{A} & \textbf{Gl} &  \textbf{Rank} \\ \hline
Gini G. Silhouette &          8 &        1 &      4 &     2 &       6 &        6 &         9 &                    3 &    5 &    8 &             6 &     6 &   11 &       8 &          2 &     5 &         \cellcolor{blue!10}5.62 \\
Gini               &          6 &        8 &      3 &     5 &       3 &        7 &        12 &                    6 &    4 &    7 &             3 &     1 &   12 &       6 &          3 &     8 &         5.88 \\
Euclidean          &         10 &        4 &      5 &     8 &       4 &       11 &         3 &                   10 &    8 &   11 &             8 &     7 &    9 &      10 &         10 &    10 &         8 \\
Manhattan          &          7 &        3 &      5 &    10 &       2 &        9 &         1 &                   10 &    6 &    9 &             9 &     4 &    8 &      11 &         10 &     6 &         6.88 \\
Minkowski          &          9 &        2 &      5 &     7 &       5 &       10 &         5 &                   10 &    7 &   10 &             7 &     5 &   10 &       9 &         10 &    11 &         7.62 \\
Cosine             &          5 &        6 &     10 &     3 &       1 &        8 &         2 &                    8 &   12 &    6 &             4 &     8 &    6 &      12 &          5 &     7 &         6.44 \\
Canberra           &          2 &        7 &      2 &     1 &       9 &        1 &        11 &                    5 &   11 &    2 &             2 &    10 &    4 &       7 &          9 &     4 &          \cellcolor{blue!20}5.44 \\
\hline
\end{tabular}
\label{tab:rankings_precision_silhouette_2}
\end{table}

\begin{table}[!htbp]
\centering
\caption{Ranking of KMeans Models by Recall: Maximizing the Silhouette Score}
\vspace{3mm}
\tiny
\renewcommand{\arraystretch}{1} 
\setlength{\tabcolsep}{1.5pt} 
\begin{tabular}{|p{0.9in}|c|c|c|c|c|c|c|c|c|c|c|c|c|c|c|c|c|}
\hline
\textbf{Distances}   &  \textbf{I} & \textbf{B} & \textbf{G} & \textbf{He} & \textbf{Ba} & \textbf{Ha} & \textbf{W} & \textbf{In} &  \textbf{Ir} & \textbf{Wi} & \textbf{Bc} & \textbf{S} & \textbf{QS} & \textbf{V} & \textbf{A} & \textbf{Gl} &  \textbf{Rank} \\ \hline
Gini G. Silhouette &         10 &        1 &      2 &     3 &       3 &        8 &        12 &                    8 &    4 &    9 &             6 &     6 &    7 &      11 &          6 &     5 &         6.31 \\
Gini               &          8 &        8 &      3 &     2 &       2 &        7 &         7 &                    6 &    4 &    8 &             3 &     2 &    8 &       6 &          4 &     6 &         \cellcolor{blue!20}5.25 \\
Euclidean          &         12 &        4 &      5 &     5 &       4 &       10 &         9 &                   10 &    7 &   11 &             8 &     8 &   11 &       8 &          9 &    10 &         8.19 \\
Manhattan          &          9 &        2 &      5 &     8 &       1 &        8 &        11 &                   10 &    4 &   10 &             9 &     4 &    9 &       9 &          9 &     8 &         7.25 \\
Minkowski          &         11 &        2 &      5 &     7 &       6 &       11 &         6 &                   10 &    7 &   11 &             7 &     7 &   12 &       7 &          9 &    10 &         8 \\
Cosine             &          5 &        5 &     11 &     4 &       5 &        6 &         2 &                    4 &   12 &    7 &             4 &     9 &    4 &      11 &          2 &     7 &         \cellcolor{blue!10}6.12 \\
Canberra           &          3 &        6 &      4 &     9 &      10 &        3 &        10 &                    2 &   10 &    3 &             1 &    11 &    2 &       9 &         12 &     4 &         6.19 \\
\hline
\end{tabular}
\label{tab:rankings_recall_silhouette_2}
\end{table}

\newpage

\subsection*{Additional experiments on agglomerative hierarchical clustering with 10\% level of noise}

\begin{table*}[!htbp]
\centering
\caption{Precision (\#P), Recall (\#R) on Agglomerative Hierarchical clustering algorithm with average linkage and Ward linkage for Euclidean distance on additional UCI datasets with 10\% level noise}
\vspace{3mm}
\makebox[\textwidth][c]{
\resizebox{\textwidth}{!}{
%
}}
\label{res_additional_agglomerative}
\end{table*}

\begin{table*}[!htbp]
\centering
\caption{Precision (\#P), Recall (\#R) on Agglomerative Hierarchical clustering algorithm with average linkage and Ward linkage for Euclidean distance on UCI datasets with 10\% level noise.}
\vspace{3mm}
\makebox[\textwidth][c]{
\resizebox{\textwidth}{!}{
%
}}
\label{res_agglomerative_hierarchical}
\end{table*}

\begin{table}[!htbp]
\centering
\caption{Ranking of Agglomerative Clustering Models by Precision: Noise 10\%}
\vspace{3mm}
\tiny
\renewcommand{\arraystretch}{1} 
\setlength{\tabcolsep}{1.5pt} 
%
\label{tab:rankings_precision_agglo_10_2}
\end{table}

\begin{table}[!htbp]
\centering
\caption{Ranking of Agglomerative Clustering Models by Recall: Noise 10\%}
\vspace{3mm}
\tiny
\renewcommand{\arraystretch}{1} 
\setlength{\tabcolsep}{1.5pt} 
%
\label{tab:rankings_recall_agglo_10}
\end{table}

\end{document}